\documentclass[letterpaper]{article} 
\usepackage{aaai24}  
\usepackage{times}  
\usepackage{helvet}  
\usepackage{courier}  
\usepackage[hyphens]{url}  
\usepackage{graphicx} 
\usepackage{natbib}  
\usepackage{caption} 
\usepackage{enumitem}
\usepackage{algorithm}
\usepackage{algorithmic}

\usepackage{booktabs}
\usepackage{multirow}  
\usepackage{array}  
\usepackage{makecell} 

\usepackage{color}
\usepackage{subfigure}

\usepackage{amsmath}  
\usepackage{amssymb}  
\usepackage{amsfonts}  
\usepackage{newfloat}
\usepackage{listings}
\DeclareCaptionStyle{ruled}{labelfont=normalfont,labelsep=colon,strut=off} 
\floatstyle{ruled}
\newfloat{listing}{tb}{lst}{}
\floatname{listing}{Listing}
\nocopyright

\title{DFedADMM: Dual Constraints Controlled Model Inconsistency\\ for Decentralized Federated Learning}
\author {    
     Qinglun Li \textsuperscript{\rm 1},
     Li Shen \textsuperscript{\rm 2}\thanks{Corresponding authors.}, 
     Guanghao Li \textsuperscript{\rm 1}, 
     Quanjun Yin \textsuperscript{\rm 1*}, 
     Dacheng Tao \textsuperscript{\rm 3}\\
}
\affiliations {
    \textsuperscript{\rm 1}  National University of Defense Technology, China \\
    \textsuperscript{\rm 2} JD Explore Academy, China \\
    \textsuperscript{\rm 3} The University of Sydney, Australia\\
    {\texttt{liqinglun@nudt.edu.cn; mathshenli@gmail.com; \\ 
    lgh@nudt.edu.cn; yin\_quanjun@163.com; dacheng.tao@gmail.com}
    }
    }

\usepackage{bibentry}

\begin{document}
	
   \maketitle
	
\begin{abstract}
To address the communication burden issues associated with federated learning (FL), decentralized federated learning (DFL) discards the central server and establishes a decentralized communication network, where each client communicates only with neighboring clients. However, existing DFL methods still suffer from two major challenges: local inconsistency and local heterogeneous overfitting, which have not been fundamentally addressed by existing DFL methods. To tackle these issues, we propose novel DFL algorithms, DFedADMM and its enhanced version DFedADMM-SAM, to enhance the performance of DFL. The DFedADMM algorithm employs primal-dual optimization (ADMM) by utilizing dual variables to control the model inconsistency raised from the decentralized heterogeneous data distributions. The DFedADMM-SAM algorithm further improves on DFedADMM by employing a Sharpness-Aware Minimization (SAM) optimizer, which uses gradient perturbations to generate locally flat models and searches for models with uniformly low loss values to mitigate local heterogeneous overfitting. Theoretically, we derive convergence rates of $\small \mathcal{O}\Big(\frac{1}{\sqrt{KT}}+\frac{1}{KT(1-\psi)^2}\Big)$ and $\small \mathcal{O}\Big(\frac{1}{\sqrt{KT}}+\frac{1}{KT(1-\psi)^2}+ \frac{1}{T^{3/2}K^{1/2}}\Big)$ in the non-convex setting for DFedADMM and DFedADMM-SAM, respectively, where $1 - \psi$ represents the spectral gap of the gossip matrix. Empirically, extensive experiments on MNIST, CIFAR10 and CIFAR100 datesets demonstrate that our algorithms exhibit superior performance in terms of both generalization and convergence speed compared to existing state-of-the-art (SOTA) optimizers in DFL.
	\end{abstract}

           \section{Introduction}\label{intro}

    Decentralized Federated Learning (DFL) is a novel distributed learning framework by adopting the principles of data sharing minimization and decentralized model aggregation without relying on a centralized server \cite{beltran2022decentralized,Sun2022Decentralized}. Specifically, DFL reduces the communication burden on servers and eliminates the risk of a single point of failure that can occur in centralized FL due to the absence of a central server \cite{gabrielli2023survey, beltran2022decentralized}. Furthermore, in the decentralized communication approach, each client communicates with its neighbors, which further minimizes the risk of privacy leakage \cite{wang2021field, li2020federated}.
	
     Nonetheless, we acknowledge that the deterioration of DFL performance is attributed to local inconsistency and the occurrence of local heterogeneous overfitting. Particularly in conventional local SGD-based DFL approaches, such as DFedAvg, the incorporation of persistent biases through local updates results in eventual divergence towards local solution discrepancies. These locally overfitted and inconsistent solutions can consequently downgrade the acquired global model parameters to an average of the locally discordant solutions. Furthermore, the phenomenon of local heterogeneous overfitting can induce pronounced client drifts, where the client gravitates toward its distinct local optimum due to heterogeneous data and local updates. This phenomenon gradually exerts an impact on global convergence, diminishing both the pace of convergence and the overall generalization capacity of the model in DFL.
	
     To address the local inconsistency during the training process and avoid client drift caused by local over-fitting, we propose the DFedADMM algorithm. Specifically, the DFedADMM algorithm utilizes alternating updates of primal and dual variables at each client to search for the saddle points of an augmented Lagrangian function. The quadratic term of the augmented Lagrangian function acts as a penalty term, ensuring that clients do not deviate too far from their initial points during each round of communication optimization. This strengthens the consistency during the training process. 
     In addition, the dual variables can capture the biases of each client during the local optimization process and make corrections after client optimization is completed. This further enhances local consistency.
     Furthermore, to address local over-fitting, we introduce gradient perturbation from sharpness-aware minimization (SAM) \cite{foret2020sharpness} to generate local flat models, yielding the enhanced version of DFedADMM, i.e., DFedADMM-SAM. This has two advantages: Firstly, gradient perturbation effectively prevents local over-fitting. Secondly, it has been suggested \cite{zhong2022improving} that generating local flat models at each client may result in a relatively flat landscape of the aggregated global model, leading to improved generalization capability.
	
     Theoretically, the convergence speed of the proposed algorithm, DFedADMM and DFedADMM-SAM, in non-convex and smooth environments has been demonstrated to be $\small\mathcal{O}\left(\frac{1}{\sqrt{KT}}+\frac{1}{KT(1-\psi)^2}\right)$ and $\small\mathcal{O}\Big(\frac{1}{\sqrt{KT}}+\frac{1}{KT(1-\psi)^2}+ \frac{1}{T^{3/2}K^{1/2}}\Big)$, respectively. We also show that better connectivity in the communication topology leads to tighter upper bounds. Empirically, we conduct extensive experiments on the MNIST, CIFAR10 and CIFAR100 datasets. The results show that our methods outperform existing baseline methods such as DPSGD, DFedAvg, DFedAvgM and DFedSAM in terms of both the generalization performance and convergence speed. In addition, our method achieves comparable or even better performance compared to some existing centralized methods such as FedAvg and FedSAM.
	
     In summary, our main contributions are four-fold:	
\begin{itemize}
	\item We propose the DFedADMM algorithm, which  enhances consistency during training by utilizing the dual variable to control the local inconsistency raised from the local heterogeneous data distribution. 
        \item  We propose the DFedADMM-SAM which further unitizes SAM to alleviate the local overfitting issue by introducing a gradient perturbation based on DFedADMM and generating flat planes for better aggregation.
	\item We derive theoretical analysis for DFedADMM and DFedADMM-SAM, which demonstrate $\mathcal{O}\Big(\frac{1}{\sqrt{KT}}+\frac{1}{KT(1-\psi)^2}\Big)$ and $\mathcal{O}\Big(\frac{1}{\sqrt{KT}}+\frac{1}{KT(1-\psi)^2}+ \frac{1}{T^{3/2}K^{1/2}}\Big)$ convergence rate of our algorithms in non-convex and decentralized environments, respectively. Theory suggests that better connectivity in the communication topology leads to tighter upper bounds.
	\item We conduct extensive experimental evaluations, which show that our methods outperform existing SOTA baseline methods in terms of generalization performance and convergence speed. Moreover, our method achieves comparable or even better performance compared to some centralized FL methods.
\end{itemize}	

        \section{Related Work}
We briefly review three lines of work that are most related to this paper, i.e., decentralized federated learning (DFL), ADMM, and sharpness-aware minimization (SAM).

\subsubsection{Decentralized Federated Learning (DFL).} DFL is the preferred learning paradigm in situations where edge devices lack trust in the central server's ability to protect their privacy \cite{yang2019federated,lalitha2018fully, lalitha2019peer} adopt a Bayesian-like approach to introduce a belief system over the model parameter space of the clients within a fully DFL framework. \cite{Sun2022Decentralized} apply momentum SGD and quantization methods for multiple local iterations to reduce communication costs, thus improving the generality of the model and providing convergence results. \cite{Rong2022DisPFL} develop a decentralized sparse training technique to further reduce communication and computational costs. \cite{shi2023improving} introduce gradient perturbations to generate local flat Models and combine them with Multiple Gossip Steps to enhance the consistency between local clients. 
More related work on DFL can be referred to the survey papers \cite{gabrielli2023survey,yuan2023decentralized,beltran2022decentralized} and their references therein.

\subsubsection{ADMM Algorithms.} Alternating Direction Method of Multipliers (ADMM) is a popular primal-dual algorithm that has achieved significant success in distributed optimization \cite{bertsekas2014constrained,gabay1976dual,boyd2011distributed}. In the field of centralized FL, FedPD \cite{fedpd} is the first to apply the primal-dual algorithm to the FL domain. It decomposes the problem into a series of local subproblems alternately updating the primal and dual variables. Under the assumption of local accuracy, it achieves a convergence rate of $\mathcal{O}(1/T)$. In addition, \cite{acar2021federated} propose a similar method called FedDyn algorithm \cite{zhang2021connection}, which employs dual variables to enhance both convergence speed and generalization performance.
Recently, several ADMM-based FL methods \cite{tran2021feddr,wang2022fedadmm,gong2022fedadmm,sun2023fedspeed,zhou2023federated,wang2023beyond} have been proposed and achieved a new SOTA performance for centralized FL.

\subsubsection{Sharpness-Aware Minimization (SAM).} 
SAM \cite{foret2020sharpness} is a powerful optimizer utilized for training deep learning models. It leverages the flat geometry of the loss landscape to enhance the generalization ability of the model. SAM and its variants have been successfully applied to a wide range of machine learning tasks \cite{zhong2022improving,zhao2022penalizing,kwon2021asam,du2021efficient} as effective optimizers. Recently, \cite{shi2023improving} applied SAM to the field of decentralized federated learning (DFL) and achieved state-of-the-art (SOTA) performance. Similarly, \cite{sun2023fedspeed,Qu2022Generalized} utilized SAM in the context of centralized federated learning (CFL) and also achieved SOTA performance. Additionally, \cite{shi2023towards} utilized SAM in the context of personalized federated learning (PFL) and achieved SOTA performance.

The most relevant work to ours is the DFedSAM  \cite{shi2023improving}. It utilizes SAM to generate locally flat landscapes, enabling the generation of potentially flat models during communication with neighboring nodes. This approach alleviates the issue of local heterogeneous overfitting, while it controls the local inconsistency via the multiple gossip steps with heavy communication costs. 
In this work, we extend the SAM optimizer and ADMM algorithm to the DFL setting to further solve the local inconsistency issue with lightweight communication cost. 
        	
\section{Methodology}\label{methodology}
	
In this part, we introduce the preliminaries and the proposed FedADMM and FedADMM-SAM algorithms  in detail.

\subsection{Problem Setup}
 Let $m$ be the total number of clients.  $T$ represents the number of communication rounds. $(\cdot)_{i,k}^{t}$ indicates variable $(\cdot)$ at the $k$-th iteration of the $t$-th round in the $i$-th client. $\mathbf{x}$ denotes the model parameters. $\mathbf{g}$ represents the stochastic gradient.  $\hat{\mathbf{g}}$ is the dual variable. The inner product of two vectors is denoted by $\langle\cdot,\cdot\rangle$, and $\Vert \cdot \Vert$ represents the Euclidean norm of a vector. As shown in Figure \ref{fig:topo}, in the decentralized network topology, the communication network between clients is represented by an undirected connected graph $\mathcal{G} = (\mathcal{N}, \mathcal{V}, {\bf W})$, where $\mathcal{N} = \{1, 2, \ldots, m\}$ denotes the set of clients, and $\mathcal{V} \subseteq \mathcal{N} \times \mathcal{N}$ represents the set of communication channels connecting pairs of distinct clients.  In the decentralized setting, there is no central server, and all clients communicate solely with their neighboring clients via the communication channels specified by $\mathcal{V}$.

   In this work, we focus on the following DFL, formulated as a finite sum of non-convex stochastic optimization:
    \begin{equation}\label{finite_sum}
		\small \min_{{\bf x}\in \mathbb{R}^d} f({\bf x}):=\frac{1}{m}\sum_{i=1}^m f_i({\bf x}),~~f_i({\bf x})=\mathbb{E}_{\xi\sim \mathcal{D}_i} F_i({\bf x};\xi),
	\end{equation}
   where $\mathcal{D}_i$ represents the data distribution in the $i$-th client in $\mathcal{N}$, which exhibits heterogeneity across clients. Each client's local objective function $F_i({\bf x};\xi)$ is associated with the training data samples $\xi$. 
Additionally, we denote $f^{*}$ as the minimum value of $f$, where $f(x) \geq f(x^*) = f^*$ for all $x \in \mathbb{R}^{d}$.
    \begin{figure*}[ht]
    \centering
    \subfigure[Ring]{
        \includegraphics[width=0.15\textwidth]{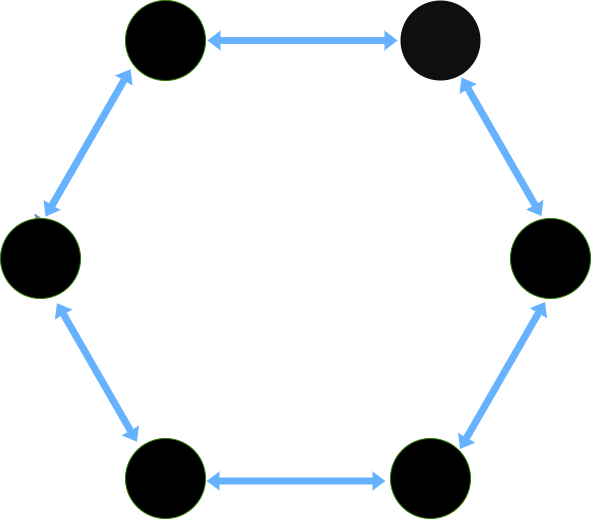}
        \label{ring}
    }\hfill
    \subfigure[Exponential]{
        \includegraphics[width=0.15\textwidth]{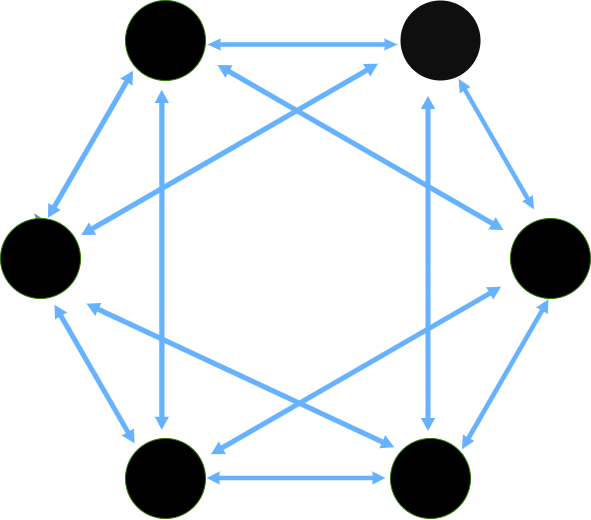}
        \label{exp}
    }\hfill
    \subfigure[Grid]{
        \includegraphics[width=0.14\textwidth]{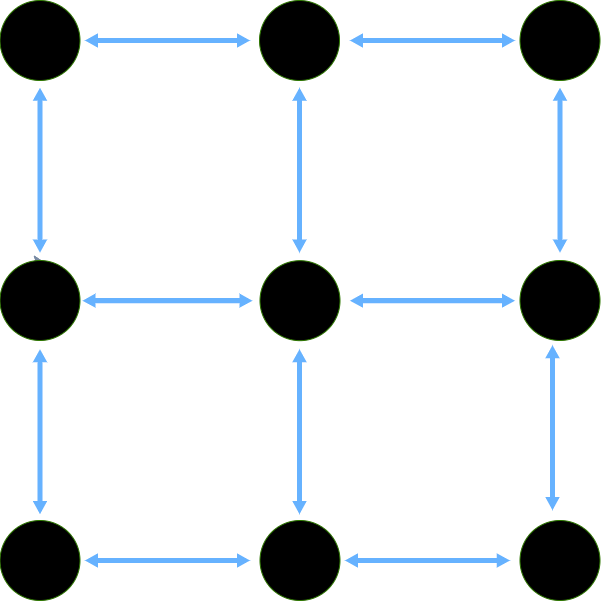}
        \label{grid}
    }\hfill
    \subfigure[Fully-connected]{
        \includegraphics[width=0.15\textwidth]{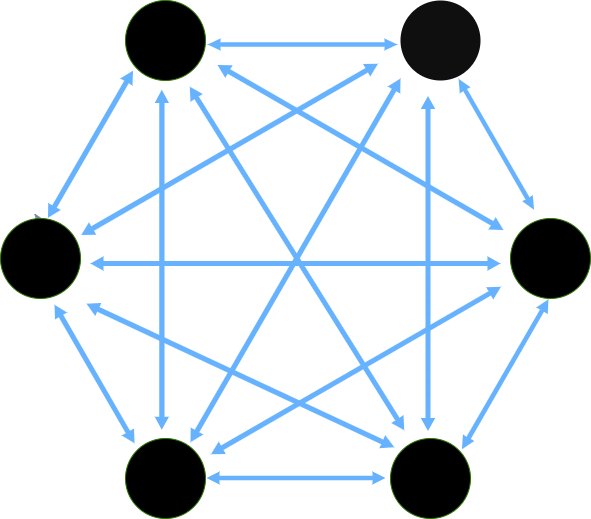}
        \label{full}
        }\hfill
    \subfigure[Random]{
        \includegraphics[width=0.15\textwidth]{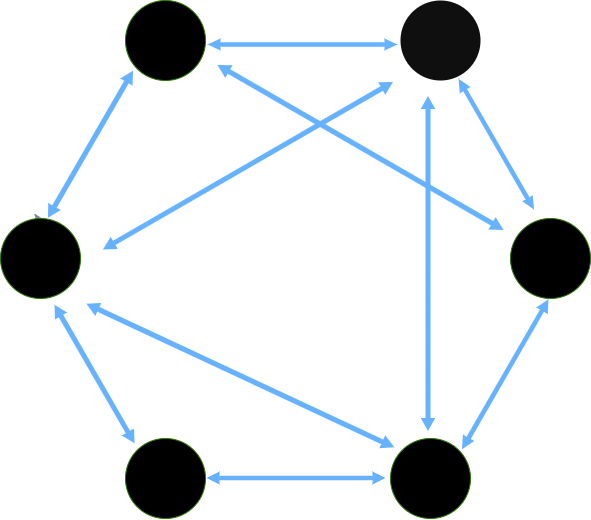}
        \label{random}
    }
    \vspace{-0.3cm}
    \caption{\small Five common decentralized communication topologies}
    \label{fig:topo}
      \vspace{-0.3cm}
    \end{figure*}

	
   \subsection{DFedADMM Algorithm}
   In this section, we extend the ADMM algorithm to the DFL setting. To facilitate the derivation of the DFedADMM algorithm, we first briefly review the FedPD \cite{fedpd} algorithm, which is a ADMM-based  centralized FL algorithm. For more detailed derivation and additional information, please refer to paper \cite{fedpd}. 

  \subsubsection{Prelimary on FedPD.}
  Firstly, the centralized optimizer of FL is based on the following global consensus formulation:
	\begin{equation}\label{finite_sum_conv}
		\min\limits_{\mathbf{x}_0,\mathbf{x}_i}\frac{1}{m}\sum\limits_{i=1}^mf_i(\mathbf{x}_i),\quad\text{s.t.}\mathbf{x}_i=\mathbf{x}_0,\forall i\in[m].
	\end{equation}
   The augmented Lagrangian (AL) function of (\ref{finite_sum_conv}) as:
	\begin{gather*}
		\mathcal{L}(\mathbf{x}_{0:m},\hat{\mathbf{g}}) \triangleq \left.\frac1m\sum_{i=1}^m\mathcal{L}_i(\mathbf{x}_0,\mathbf{x}_i,\hat{\mathbf{g}}_i)\right. \\		\mathcal{L}_i(\mathbf{x}_i,\mathbf{x}_0,\hat{\mathbf{g}}_i) 
		\triangleq \left.f_i(\mathbf{x}_i) - \langle \hat{\mathbf{g}}_i,\mathbf{x}_i-\mathbf{x}_0\rangle+\frac1{2\lambda}\left\|\mathbf{x}_i-\mathbf{x}_0\right\|^2.\right.
	\end{gather*}
     The main idea of the FedPD algorithm is to solve the optimization problem of the local AL $\mathcal{L}_i$ function using ADMM. The updates for the primal and dual variables of each client in FedPD are as follows:
     \begin{align}	\!\!\mathbf{x}_{i}^{t+1}\label{x_update}\!=\!\arg\min_{\mathbf{x}}f_i(\mathbf{x})-&\left\langle\hat{\mathbf{g}}_i^t,\mathbf{x}-\mathbf{x}_{0}^t\right\rangle+\frac{1}{2\lambda}\left\|\mathbf{x}-\mathbf{x}_{0}^t\right\|^2\!\!  \\	
\hat{\mathbf{g}}_{i}^{t+1}\label{g_update} =\hat{\mathbf{g}}_i^t - &\frac{1}{\lambda}(\mathbf{x}_i^{t+1}-\mathbf{x}_{0}^t) 
	\end{align}
In the case where all clients are fully participating, the updates for the server in the FedPD algorithm are as follows:
	\begin{align}
        \label{average}
		\mathbf{x}_0^{t+1}=\dfrac{1}{m}\sum_{i=1}^m(\mathbf{x}_i^{t+1}-\lambda\hat{\mathbf{g}}_i^{t+1})
	\end{align}
 Equation (\ref{average}) indicates that the server is responsible for communication, receiving the model parameters from participating client devices, and performing the aggregation. This can lead to a communication burden, and if the server crashes, the entire system's updates will be interrupted. Below, we try to reallocate communication to each client in a peer-to-peer manner by discarding the central sever to address these issues.
 
  \subsubsection{Derive the DFedADMM.} To mitigate the risks of single point of failure and communication burden associated with centralization, we  make three modifications to the above update scheme to accommodate a decentralized topology :
\begin{itemize}[leftmargin=*]
\item Change in solving method for each client: For each client, the precise solution of Equation (\ref{x_update}) is computationally expensive. To better align with practical scenarios, we can modify Equation (\ref{x_update}) to perform $ K $ rounds of local stochastic gradient descent, as described in line 13 of Algorithm \ref{Combined-DFedADMM}.
The specific update formula for $\mathbf{x}_{i,k}^t$ is as follows:
	\begin{equation}\label{x_update_d}
		\mathbf{x}_{i,k+1}^{t}=\mathbf{x}_{i,k}^{t}-\eta_{l}\left(\mathbf{g}_{i,k}^{t} - \hat{\mathbf{g}}_{i}^{t-1} + \frac{1}{\lambda}\left(\mathbf{x}_{i,k}^{t}-\mathbf{x}_i^{t}\right)\right)
	\end{equation}
	where $\mathbf{g}_{i,k}^{t}$ is the random gradient, $\hat{\mathbf{g}}_{i}^{t-1}$ is the dual variable, $\lambda$ is the penalty parameter, and $\mathbf{x}_i^{t}$ is the model parameter of the $i$-th client after communication with neighbors in the $t$-th round. After performing $K$ local updates, we obtain the following equation: 
	\begin{equation}\label{local_update}
		\mathbf{x}_{i,K}^{t}-\mathbf{x}_{i}^{t} =-\lambda\gamma\sum_{k=0}^{K-1}\frac{\gamma_{k}}{\gamma}\mathbf{g}_{i,k}^{t} + \gamma\lambda\hat{\mathbf{g}}_{i}^{t-1}
	\end{equation}
	where $\sum_{k=0}^{K-1}\gamma_{k}=\sum_{k=0}^{K-1}{\eta_{l}}/{\lambda}\bigl(1-{\eta_{l}}/{\lambda}\bigr)^{K-1-k}=\gamma=1-(1-{\eta_{l}}/{\lambda})^{K}$. Due to the space limitation, proof details of above equation is placed in the \textbf{Appendix}.
	
	\item Change in initial values for each client: Without the coordination and communication from a central server to transmit initial model values, the initial values for each client's updates, denoted as $\mathbf{x}_0^t$, need to be modified to be specific to each client, denoted as $\mathbf{x}_i^t$. The parameter $\mathbf{x}_i^t$ will be stored locally on client $i$.
	
    \item Change in communication scheme: In the absence of server aggregation, the previous server communication followed by aggregation is replaced with individual peer-to-peer communication via the communication topology matrix $\bf W$. After communication with their respective neighbors, each client performs a local aggregation operation. The specific communication expression is as follows:
	\begin{equation}\label{comm_update}
		\mathbf{x}_i^{t+1}= \sum_{l \in \mathcal{N}\left( i \right)} w_{il} \mathbf{z}^{t}_{i}
	\end{equation}
	where $\mathbf{x}_i^{t+1}$ represents the updated model parameter of the $i$-th client after communication with its neighbors, $\mathcal{N}\left( i \right)$ represents the set of neighbors of the $i$-th client, $w_{il}$ represents the weight or communication coefficient between the $i$-th client and its neighbor $l$, and the update method for $\mathbf{z}^{t}_{i}$ can be found in line 17 of Algorithm \ref{Combined-DFedADMM}.
\end{itemize}
The complete algorithm is summarized in Algorithm \ref{Combined-DFedADMM}.

    \begin{algorithm}[t]
    	\small
    	\caption{DFedADMM and DFedADMM-SAM}
    	\label{Combined-DFedADMM}
     	\renewcommand{\algorithmicrequire}{\textbf{Input:}}
		\renewcommand{\algorithmicensure}{\textbf{Output:}}
    	\begin{algorithmic}[1]
    		\REQUIRE model parameters $\mathbf{x}_i^{0}$, total communication rounds $T$, local gradient controller $\hat{\mathbf{g}}_{i}^{-1}=0$, penalized weight $\lambda$.
    		\ENSURE model average parameters $\bar{\mathbf{x}}^{t}$.
    		\FOR{$t = 0, 1, 2, \cdots, T-1$}
    		\FOR{client $i$ in parallel}
    		\STATE set $\mathbf{x}_{i,0}^{t}=\mathbf{x}_i^{t}$
    		\FOR{$k = 0, 1, 2, \cdots, K-1$}
    		\STATE sample minibatch $\varepsilon_{i,k}^{t}$ and compute unbiased stochastic gradient: $\mathbf{g}_{i,k}^{t}=\nabla f_{i}(\mathbf{x}_{i,k}^{t};\varepsilon_{i,k}^{t})$
    		
    		\IF{\textcolor{blue}{DFedADMM}}
    		\STATE update gradient descent step: $\mathbf{x}_{i,k+1}^{t}\gets \mathbf{x}_{i,k}^{t} - \eta_{l}\bigl(\mathbf{g}_{i,k}^{t} - \hat{\mathbf{g}}_{i}^{t-1} + \frac{1}{\lambda}(\mathbf{x}_{i,k}^{t}-\mathbf{x}_i^{t})\bigr)$
    		\ENDIF
    		
    		\IF{\textcolor{blue}{DFedADMM-SAM}}
    		\STATE compute unbiased stochastic gradient: $\mathbf{g}_{i,k,1}^{t}=\nabla f_{i}(\mathbf{x}_{i,k}^{t};\varepsilon_{i,k}^{t})$
    		\STATE update extra step: $\Breve{\mathbf{x}}_{i,k}^{t}=\mathbf{x}_{i,k}^{t}+\rho\frac{\mathbf{g}_{i,k,1}^{t}}{\Vert\mathbf{g}_{i,k,1}\Vert}$
    		\STATE compute unbiased stochastic gradient: $\mathbf{g}_{i,k}^{t}=\nabla f_{i}(\Breve{\mathbf{x}}_{i,k}^{t};\varepsilon_{i,k}^{t})$
    		\STATE update gradient descent step: $\mathbf{x}_{i,k+1}^{t}\gets \mathbf{x}_{i,k}^{t} - \eta_{l}\bigl(\mathbf{g}_{i,k}^{t} - \hat{\mathbf{g}}_{i}^{t-1} + \frac{1}{\lambda}(\mathbf{x}_{i,k}^{t}-\mathbf{x}_i^{t})\bigr)$
    		\ENDIF
    		
    		\ENDFOR
    		\STATE $\hat{\mathbf{g}}_{i}^{t}=\hat{\mathbf{g}}_{i}^{t-1}-\frac{1}{\lambda} (\mathbf{x}_{i,K}^{t}-\mathbf{x}_i^{t})$
    		\STATE $\mathbf{z}_{i}^{t}=\mathbf{x}_{i,K}^{t}-\lambda\hat{\mathbf{g}}_{i}^{t-1}$ 
    		\STATE Receive neighbors' models $\mathbf{z}_{l}^{t}$ with adjacency matrix ${\bf W}$
    		\STATE $\mathbf{x}_i^{t+1}=\sum_{l \in \mathcal{N}\left( i \right)} w_{il} \mathbf{z}^{t}_{i}$
    		\ENDFOR
    		\ENDFOR
    	\end{algorithmic}
    \end{algorithm}
	
\subsection{DFedADMM-SAM Algorithm}
Based on the SAM optimizer, we propose the DFedADMM-SAM algorithm, which builds upon the DFedADMM to further alleviate the issue of local overfitting. DFedADMM-SAM algorithm can be described as follows. Each client maintains a local model $\mathbf{x}$ in its local memory and repeats the following steps (taking client $i$ as an example):
    \begin{itemize}
        \item Client $i$ samples a minibatch $\varepsilon_{i,k}^{t}$ from the data distribution $\mathcal{D}_i$ and computes the unbiased stochastic gradient: $\mathbf{g}_{i,k,1}^{t}=\nabla f_{i}(\mathbf{x}_{i,k}^{t};\varepsilon_{i,k}^{t})$.
	
	\item Client $i$ performs a gradient ascent step at $\mathbf{x}_{i,k}^{t}$: $\Breve{\mathbf{x}}_{i,k}^{t}=\mathbf{x}_{i,k}^{t}+\rho\frac{\mathbf{g}_{i,k,1}^{t}}{\Vert\mathbf{g}_{i,k,1}\Vert}$, where $\rho$ is the step size.
	
	\item Client $i$ uses minibatch $\varepsilon_{i,k}^{t}$ at $\Breve{\mathbf{x}}_{i,k}^{t}$ to compute the unbiased gradient $\mathbf{g}_{i,k}^{t}$.
    \end{itemize}

  The updates for $\hat{\mathbf{g}}_{i}^{t}$, $\mathbf{z}_{i}^{t}$, and $\mathbf{x}_i^{t+1}$ in the subsequent steps are the same as DFedADMM algorithm. The complete algorithm scheme can be found in Algorithm \ref{Combined-DFedADMM}.

        \newtheorem{theorem}{Theorem}
\newtheorem{proposition}{Proposition}
\newtheorem{lemma}{Lemma}
\newtheorem{corollary}{Corollary}
\newtheorem{definition}{Definition}
\newtheorem{assumption}{Assumption}
\newtheorem{remark}{Remark}
\newtheorem{proof}{Proof}
\section{Convergence Analysis}\label{analysis}

In this section, we present the theoretical analysis for FedADMM and FedADMM-SAM algorithms. Due to space constraints,  detailed proofs are placed in the \textbf{Appendix}.  Below, we provide some standard definitions and assumptions.
	
	\begin{definition}\label{mixing matrix}
		(Gossip/Mixing matrix). [Definition 1, \cite{Sun2022Decentralized}] The gossip matrix ${ \bf W} = [w_{i,j}] \in [0,1]^{m\times m}$  is assumed to have the following properties:
		(i) \textbf{(Graph)} If $i\neq j$ and $(i,j) \notin {\cal V}$, then $w_{i,j} =0$, otherwise, $w_{i,j} >0$;
		(ii) \textbf{(Symmetry)} ${\bf W} = {\bf W}^{\top}$;
		(iii) \textbf{(Null space property)} $\mathrm{null} \{{\bf I}-{\bf W}\} = \mathrm{span}\{\bf 1\}$;
		(iv) \textbf{(Spectral property)} ${\bf I} \succeq {\bf W} \succ -{\bf I}$. 
		Under these properties, the eigenvalues of $\bf{\bf W}$ satisfy $1=|\psi_1({\bf W)})|> |\psi_2({\bf W)})| \ge \dots \ge |\psi_m({\bf W)})|$. Furthermore, we define $\psi:=\max\{|\psi_2({\bf W)}|,|\psi_m({\bf W)})|\}$ and $1-\psi \in (0,1]$ as the spectral gap of $\bf W$.
	\end{definition}
	
	\begin{assumption}\label{smoothness}
		(\textbf{L-Smoothness}) \textit{The non-convex function $f_{i}$ satisfies the smoothness property for all $i\in[m]$, i.e., $\Vert\nabla f_{i}(\mathbf{x})-\nabla f_{i}(\mathbf{y})\Vert\leq L\Vert\mathbf{x}-\mathbf{y}\Vert$, for all $\mathbf{x},\mathbf{y}\in\mathbb{R}^{d}$.}
	\end{assumption}
	\begin{assumption}\label{bounded_stochastic_gradient}
    (\textbf{Bounded Stochastic Gradient}) The stochastic gradient $\mathbf{g}_{i,k}^{t}=\nabla f_{i}(\mathbf{x}_{i,k}^{t}, \varepsilon_{i,k}^{t})$ with the randomly sampled data $\varepsilon_{i,k}^{t}$ on the local client $i$ is unbiased and with bounded variance, i.e., $\mathbb{E}[\mathbf{g}_{i,k}^{t}]=\nabla f_{i}(\mathbf{x}_{i,k}^{t})$ and $\mathbb{E}\Vert \mathbf{g}_{i,k}^{t} - \nabla f_{i}(\mathbf{x}_{i,k}^{t})\Vert^{2} \leq \sigma_{l}^{2}$, for all $\mathbf{x}_{i,k}^{t}\in\mathbb{R}^{d}$.
     \end{assumption}
	
	\begin{assumption}\label{bounded_heterogeneity}
   (\textbf{Bounded Heterogeneity}) 
  The dissimilarity of the dataset among the local clients is bounded by the local and global gradients, i.e., $\mathbb{E}\Vert\nabla f_{i}(\mathbf{x})-\nabla f(\mathbf{x})\Vert^{2}\leq\sigma_{g}^{2}$, for all $\mathbf{x}\in\mathbb{R}^{d}$.
  This paper also assumes global variance is  bounded, i.e., $\frac{1}{m} \sum_{i=1}^m \|\nabla f_i({\bf x}) - \nabla f({\bf x})\|^2 \leq \sigma_{g}^2$.
\end{assumption}
	
	\begin{assumption}\label{Bounded gradient}
		(\textbf{Bounded gradient}) we have $ \|\nabla f_i(\mathbf{x})\|\leq B. $ for any $ i \in \{1,2,\cdots,m\} $.
	\end{assumption}
	
Note that the above-mentioned assumptions are mild and commonly used in characterizing the convergence rate of DFL \cite{Sun2022Decentralized,shi2023make}.
	
Compared with decentralized parallel SGD methods such as D-PSGD \cite{lian2017can}, the technical difficulty stems from the realization that $\mathbf{x}_i^{t,K}-\mathbf{x}_i^t$ cannot be regarded as an unbiased estimate of the gradient $\nabla f_i(\mathbf{x}_i^t)$ following numerous local iterations \cite{Sun2022Decentralized,shi2023improving}. Moreover, in our algorithm (Algorithm \ref{Combined-DFedADMM}, line 13), the introduction of dual variables adds additional challenges to the analysis of $\mathbf{x}_i^{t,K}-\mathbf{x}_i^t$.
To overcome this problem, we define two auxiliary sequences: $\overline{\mathbf{x}^{t}}=\frac{1}{m}\sum_{i\in[m]}\mathbf{x}_{i}^{t}$ and $\mathbf{w}^{t}=\overline{\mathbf{x}^{t}}+\frac{1-\gamma}{\gamma}(\overline{\mathbf{x}^{t}}-\overline{\mathbf{x}^{t-1}})$. By doing so, we construct a mapping from $\mathbf{x}^{t}$ to $\mathbf{w}^{t}$, and the entire update process is simplified to an SGD-type method with gradients $\mathbf{g}$. The detailed proof can be found in the \textbf{Appendix}.

	\begin{theorem}\label{convergence1}
		Under the Assumptions \ref{smoothness}-\ref{Bounded gradient}, when the local learning rate $\eta_{l}$ satisfies $\eta_{l}\leq\min\{\frac{1}{96KL},2\lambda\}$, and the local interval $K$ satisfies $K \geq \lambda/\eta_{l}$, let $\kappa=\frac{1}{2} - 1152L^2\eta_{l}^2K^2 $ be a positive constant selected with the proper values of $\eta_{l}$ and $\rho$, then the auxiliary sequence $\mathbf{w}^{t}$, generated by executing Algorithm \ref{Combined-DFedADMM} of DFedADMM-SAM, satisfies:
		\begin{equation*}
			\small
			\begin{split}
				&\frac{1}{T}\sum_{t=0}^{T-1}\mathbb{E}_{t}\Vert\nabla f(\mathbf{w}^{t})\Vert^{2}	\\
				& \leq \frac{f(\mathbf{w}^{0}) - f^*}{\lambda \kappa T} + \frac{96 L^2\eta_l^2K^2}{\kappa}(3\sigma_g^2 + 2\sigma_l^2 + B^2)  \\
				& + \frac{3\lambda ^2  L^2(1 + 384L^2\eta_{l}^2K^2)}{\kappa} (\sigma_{l}^2 + B^2)+ \frac{L\lambda}{2\kappa}\left( \sigma_l^2 + B^2\right) \\
				& + \frac{96 \eta_{l}^2K^2L^2 \left(1+192L^2\eta_{l}^2 K^2\right)\left(3\sigma_g^2 +2\sigma_{l}^2 +4B^2\right)}{\kappa (1-\psi)^2}\\
				& + \frac{6 \lambda^2L^2 \left(1+192L^2\eta_{l}^2 K^2\right)\left(\sigma_{l}^2 + B^2\right)}{\kappa (1-\psi)^2} +\frac{6\rho^2\lambda L^2}{\kappa} (\sigma_{l}^2 + B^2)\\
			\end{split}    
		\end{equation*}
		where $f$ is a non-convex objective function $f^{*}$ is the optimal of $f$, T is the number of communication rounds. More specifically, by setting $\rho=0$, we obtain the convergence theorem for Algorithm \ref{Combined-DFedADMM} of DFedADMM.
	\end{theorem}
	More detailed proof can be found in the \textbf{Appendix}. With Theorem \ref{convergence1}, we can state the following convergence rates for DFedADMM and DFedADMM-SAM algorithms:
	\begin{corollary}\label{coro_DFedADMM}
		Let the local adaptive learning rate satisfy $\eta_l=\mathcal{O}(\frac{1}{LK\sqrt{T}})$ , the penalty parameter   satisfy $\lambda=\mathcal{O}(\frac{1}{L\sqrt{KT}})$, and the perturbation parameter $\rho = \mathcal{O}(\frac{1}{\sqrt{T}})$. Then, the convergence rate for DFedADMM-SAM satisfies:
		\small
		\begin{align*}
			&\frac{1}{T}\sum_{t=0}^{T-1}\mathbb{E}_{t}\Vert\nabla f(\mathbf{w}^{t})\Vert^{2}  = \mathcal{O} \Big( \frac{(f({\bf w}^{0})-f^{*}) +K^{-1/2}(B^2+\sigma_l^2)}{\sqrt{T}}\\
			&+\frac{\left(B^2 +\sigma_{l}^2 +\sigma_{g}^2 \right)}{T(1-\psi)^2} + \frac{L(B^2+\sigma_l^2)}{T^{3/2}K^{1/2}}  +\frac{(B^2+\sigma_l^2)}{TK(1-\psi)^2}\Big).   
		\end{align*}
		More specifically, by setting $\rho=0$, we obtain the convergence rate for DFedADMM:
		\begin{align*}
			&\frac{1}{T}\sum_{t=0}^{T-1}\mathbb{E}_{t}\Vert\nabla f(\mathbf{w}^{t})\Vert^{2}  = \mathcal{O} \Big( \frac{(f({\bf w}^{0})-f^{*}) +K^{-1/2}(B^2+\sigma_l^2)}{\sqrt{T}}\\
			&+\frac{\left(B^2 +\sigma_{l}^2 +\sigma_{g}^2 \right)}{T(1-\psi)^2} +\frac{(B^2+\sigma_l^2)}{TK(1-\psi)^2}\Big).
		\end{align*}
	\end{corollary}

	\begin{remark}
		From Corollary \ref{coro_DFedADMM}, we observe that the convergence speed of the DFedADMM and DFedADMM-SAM algorithms improves with an increase in the number of local iterations, K. This finding is consistent with the results obtained by \cite{Sun2022Decentralized}. When K is large enough, For DFedADMM, the term $\mathcal{O}\left(\frac{1}{\sqrt{T}}+\frac{B^2 +\sigma_{l}^2 +\sigma_{g}^2 }{T(1-\psi)^2}\right)$ will dominate the convergence bound. Similarly, for DFedADMM-SAM, the term $\mathcal{O}\Big(\frac{1}{\sqrt{T}}+\frac{B^2 +\sigma_{l}^2 +\sigma_{g}^2 }{T(1-\psi)^2}+ \frac{1}{T^{3/2}}\Big)$ will dominate the convergence bound. To achieve an error of $\epsilon > 0$, Both require $ \mathcal{O}\left(\frac{1}{\epsilon^2}\right) $ communication rounds, which is the same as SGD and DSGD. When $\psi$ is smaller, the convergence bounds of both algorithms will be tighter. This means that when the communication topology is better connected, tighter convergence bounds can be achieved.
	\end{remark}
	In summary, both the DFedADMM algorithm and the DFedADMM-SAM algorithm achieve a convergence rate of $ 1/\sqrt{T} $. The difference between them lies in the additional term $\frac{L(B^2+\sigma_l^2)}{T^{3/2}K^{1/2}}$, which arises from the additional SGD step for smoothness using the SAM local optimizer. This term can be considered negligible compared to the other terms.

\section{Experiments}
 In this section, we evaluate the efficacy of our algorithms compared to six baselines from CFL and DFL settings. In addition, we conduct several experiments to verify the impact of the communication network topology. 

\subsection{Experiment Setup}
\subsubsection{Dataset and Data Partition.} The effectiveness of the proposed DFedADMM and DFedADMM-SAM methods is evaluated on the MNIST, CIFAR-10, and CIFAR-100 datasets \cite{krizhevsky2009learning} in both IID and non-IID settings. To simulate non-IID data distribution among federated clients, the Dirichlet Partition \cite{hsu2019measuring} approach is utilized. This approach partitions the local data of each client by sampling label ratios from the Dirichlet distribution Dir($\alpha$), where parameters $\alpha=0.3$ and $\alpha=0.6$ are used for this purpose.

\subsubsection{Baselines.} The baselines used for comparison include several state-of-the-art (SOTA) methods in both the CFL and DFL settings. In the centralized setting, the baselines consist of FedAvg \cite{mcmahan2017communication} and FedSAM \cite{Qu2022Generalized}. For the decentralized setting, the baselines include D-PSGD \cite{lian2017can}, DFedAvg, DFedAvgM \cite{Sun2022Decentralized}, and DFedSAM \cite{shi2023improving}. These baselines are selected to provide comprehensive comparisons across different settings and methodologies.

\subsubsection{Implementation Details.} The total number of clients is set to 100, with 10\% of the clients participating in the communication. For decentralized methods, all clients perform the local iteration step, while for centralized methods, only the participating clients perform the local update. The local learning rate is initialized to 0.1 with a decay rate of 0.998 per communication round for all experiments. For SAM-based algorithms, such as DFedSAM and DFedADMM-SAM, we set the gradient perturbation weight $\rho = 0.1$. For the CIFAR-100 dataset, we adopt ResNet-18 \cite{simonyan2014very} as the backbone architecture for each client. Please refer to the \textbf{Appendix} for the specific backbone architectures used for MNIST and CIFAR-10. The number of communication rounds is set to 500 for all experiments comparing baselines on CIFAR-10/100, and 300 rounds for MNIST. For investigating the topology-aware performance, we employ a data partitioning scheme on the MNIST dataset using Dirichlet $\alpha=0.3$. Additionally, all ablation studies are conducted on the CIFAR-10 dataset with 300 communication rounds.

\subsubsection{Communication Configurations.} 
To ensure a fair comparison between decentralized and centralized approaches, we have implemented a dynamic and time-varying connection topology for the decentralized methods. This guarantees that the number of connections in each round does not exceed that of the central server. By controlling the number of client-to-client communications, we can match the communication volume of the centralized methods. It is worth noting that, based on previous research, the communication complexity is measured in terms of the frequency of local communications. For detailed experimental information, please refer to \textbf{Appendix} due to space constraints.

\subsection{Performance Evaluation}\label{exper-evaluation}
\subsubsection{Comparative performance analysis with baseline methods.} We assess the performance of DFedADMM and DFedADMM-SAM with $\rho=0.1$ on MNIST and CIFAR-10/100 datasets. We compare these methods with all baselines from CFL and DFL in both experimental settings. From the results in Figure \ref{fig:Compared_baselines} (found in the \textbf{Appendix}) and Table \ref{ta:all_baselines}, it is evident that our proposed algorithm, DFedADMM-SAM, achieves superior accuracy and convergence speed on the test sets compared to all baseline methods on the CIFAR-10/100 datasets. It achieves state-of-the-art (SOTA) performance. Even in the MNIST IID setting, our algorithm performs comparably to the CFL method. In terms of convergence speed, our proposed algorithm achieves convergence in fewer communication rounds. Specifically, the speed of the DFedAvgM method is comparable to our proposed algorithm in the early stages, but its accuracy is lower than our algorithm in the middle and later stages. Specifically, it reaches convergence in approximately $2/5$ of the specified communication rounds, and during this period, our algorithm demonstrates a noticeable performance gap compared to other baseline methods. This demonstrates the fast convergence speed of our algorithm (For further detailed discussion on the convergence speed, please refer to the \textbf{Appendix}), which also verifies that the introduced dual variable can control the model inconsistency well for DFL.

\noindent
\subsubsection{Impact of non-IID levels.} Our algorithm demonstrates certain robustness to heterogeneous data. In Table \ref{ta:all_baselines}, we can observe that our algorithm exhibits robustness to different data partition scenarios, including IID, Dirichlet 0.6, and Dirichlet 0.3. The presence of heterogeneous data makes training the global/consensus model more challenging. However, our algorithm achieves the highest accuracy across all data partition scenarios. Moreover, as the data heterogeneity increases (with $\alpha$ decreasing from 1 to 0.3, where 1 represents the IID scenario), our algorithm maintains minimal accuracy loss compared to baseline methods. For example, On the CIFAR-10 dataset, the difference in accuracy obtained by our algorithm between IID and Dirichlet 0.3 is 1.92\%, outperforming FedAvg (3.05\%), FedSAM (2.99\%), DFedSAM (2.41\%), and DFedAvgM (3.04\%). This result highlights the effectiveness of our algorithm in mitigating the impact of heterogeneous data on the global/consensus model through the introduction of dual constraints control.
\begin{table*}[ht]  
    \footnotesize  
    \centering 
    \vspace{-0.3cm}
    \renewcommand{\arraystretch}{0.95}  
    \begin{tabular}{cccc!{\vrule width \lightrulewidth}ccc!{\vrule width \lightrulewidth}ccc}  
        \toprule  
        \multirow{2}{*}{Algorithm}& \multicolumn{3}{c!{\vrule width \lightrulewidth}}{MNIST} & \multicolumn{3}{c!{\vrule width \lightrulewidth}}{CIFAR-10} & \multicolumn{3}{c}{CIFAR-100} \\  
        \cmidrule{2-10}  
        & \multicolumn{1}{c}{Dir 0.3} & \multicolumn{1}{c}{Dirt 0.6} & \multicolumn{1}{c!{\vrule width \lightrulewidth}}{IID} &
        \multicolumn{1}{c}{Dir 0.3} & \multicolumn{1}{c}{Dir 0.6} & \multicolumn{1}{c!{\vrule width \lightrulewidth}}{IID} & \multicolumn{1}{c}{Dirt 0.3} & \multicolumn{1}{c}{Dir 0.6} & \multicolumn{1}{c}{IID} \\    
        \midrule  
        FedAvg & 97.33 & 98.10 & 98.29     & 78.64 & 79.71 & 81.69 
        & 55.11 & 56.01   & 57.40 \\ 
        
        FedSAM  & 98.32 & 98.36 & \textbf{98.52}    & 80.56 & 81.55 & 83.55 
        & 56.52 & 57.39   & 57.79  \\    
        
        
        \midrule  
        
        D-PSGD  & 97.26 & 97.49 & 97.67     & 73.02 & 74.68 & 78.06 &
        56.17  & 56.61  & 57.39  \\    
        
        DFedAvg  & 97.78 & 97.74 & 98.07     & 77.56 & 78.59 & 79.42 &
        59.11 & 59.46  & 60.22 \\    
        
        DFedAvgM & 98.16 & 98.21 & 98.25     & 79.91 & 80.77 & 82.95 &
        54.72 & 55.93  & 56.24  \\    
        
        DFedSAM  & 98.34 & 98.21 & 98.31      & 80.10 & 80.82 & 82.44 &
        59.19 & 60.03  & 60.84  \\   
        
        DFedADMM & 98.07 & 98.07 & 98.09      & 80.50 & 80.29 & 80.05 & 
        58.83 & 58.89 & 60.80  \\   
        
        DFedADMM-SAM  & \textbf{98.44} & \textbf{98.47} & \textbf{98.44} 
        & \textbf{82.00} & \textbf{82.36} & \textbf{83.92} &
        \textbf{59.37} & \textbf{60.44} &  \textbf{61.68} \\ 
        \bottomrule  
    \end{tabular}  
    \captionsetup{position=bottom}
    \vspace{-0.2cm}
    \caption{\small Top 1 test accuracy (\%) on two datasets in both IID and non-IID settings.}  
    \label{ta:all_baselines}  
    \vspace{-0.2cm}
\end{table*}

\subsection{Topology-aware Performance}\label{topoaware}

\begin{table}[ht]
    \centering
    \renewcommand{\arraystretch}{1}
    \resizebox{0.9\linewidth}{!}{
        \begin{tabular}{ccccc} 
            \toprule
            \multicolumn{1}{c}{Algorithm} & \multicolumn{1}{c}{Ring} & \multicolumn{1}{c}{Grid} & \multicolumn{1}{c}{Exp} & \multicolumn{1}{c}{Full}  \\ 
            \midrule
            D-PSGD        & 93.31 & 94.17 & 93.48 & 95.02           \\
            DFedAvg       & 95.86 & 97.49 & 97.57 & 98.00           \\
            DFedAvgM      & 95.62 & 97.76 & 98.07 & 98.17           \\
            DFedSAM       & 96.86 & 97.98 & 98.01 & 98.13           \\
            DFedADMM      & 96.41 & 97.94 & 97.99 & 98.26           \\
            DFedADMM-SAM  & \textbf{97.08} & \textbf{98.27} & \textbf{98.29} & \textbf{98.50}           \\
            \bottomrule
        \end{tabular}
    }
    \captionsetup{position=bottom} 
    \vspace{-0.2cm}
    \caption{\small Top 1 test accuracy (\%) in various network topologies compared with decentralized algorithms on MNIST.}
    \label{ta:topo}
    \vspace{-0.2cm}
\end{table}

In this section, we examine the effects of different DFL communication topologies on various DFL methods using the MNIST dataset. Unlike in Section \ref{exper-evaluation}, the communication topology in this section is deterministic rather than random. Each client in the network only communicates with its predetermined neighbors, and the specific communication pattern is determined by the corresponding topology as illustrated in Figure \ref{fig:topo}.
In DFL, the degree of sparse connectivity $\psi$ follows the order: Ring $>$ Grid $>$ Exponential (Exp) $>$ Full-connected (Full) \cite{shi2023improving,zhu2022topology}.

From Figure \ref{fig:topo_aware} and Table \ref{ta:topo}, it can be observed that as the sparse connectivity $\psi$ decreases, the accuracy of our proposed algorithm on the test set increases. This is attributed to the fact that a well-designed communication topology allows clients to obtain better optimization starting points through communication, resulting in improved results. Furthermore, our algorithm consistently achieves higher test set accuracy compared to other DFL baselines across different communication topologies. This result showcases the effectiveness of our introduced dual constraints control in mitigating model inconsistency.

\begin{figure}[h]
    \centering
    \includegraphics[width=0.32\textwidth]{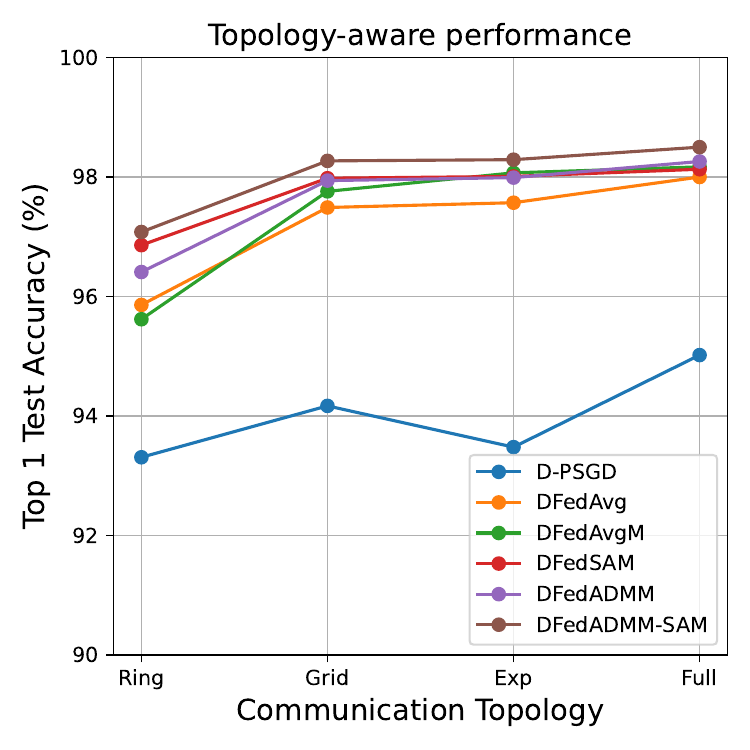}
     \vspace{-0.3cm}
 \caption{\small Topology-aware performance of DFL methods.}
    \label{eig}
      \vspace{-0.3cm}
\end{figure}

\subsection{ Ablation Study}
We conduct experiments to evaluate the impact of each component and hyper-parameter in DFedADMM-SAM, where $\rho$ was set to 0.1, except for the study of $\rho$ itself. All experiments are conducted with the "Random" topology, where each client communicates with only 10 randomly selected neighbors among 100 total clients in each round. The performance analysis on different topologies has been addressed in the Topology-aware Performance section.

\begin{figure}[ht]
\begin{center}
\subfigure{
    	\includegraphics[width=0.44\textwidth]{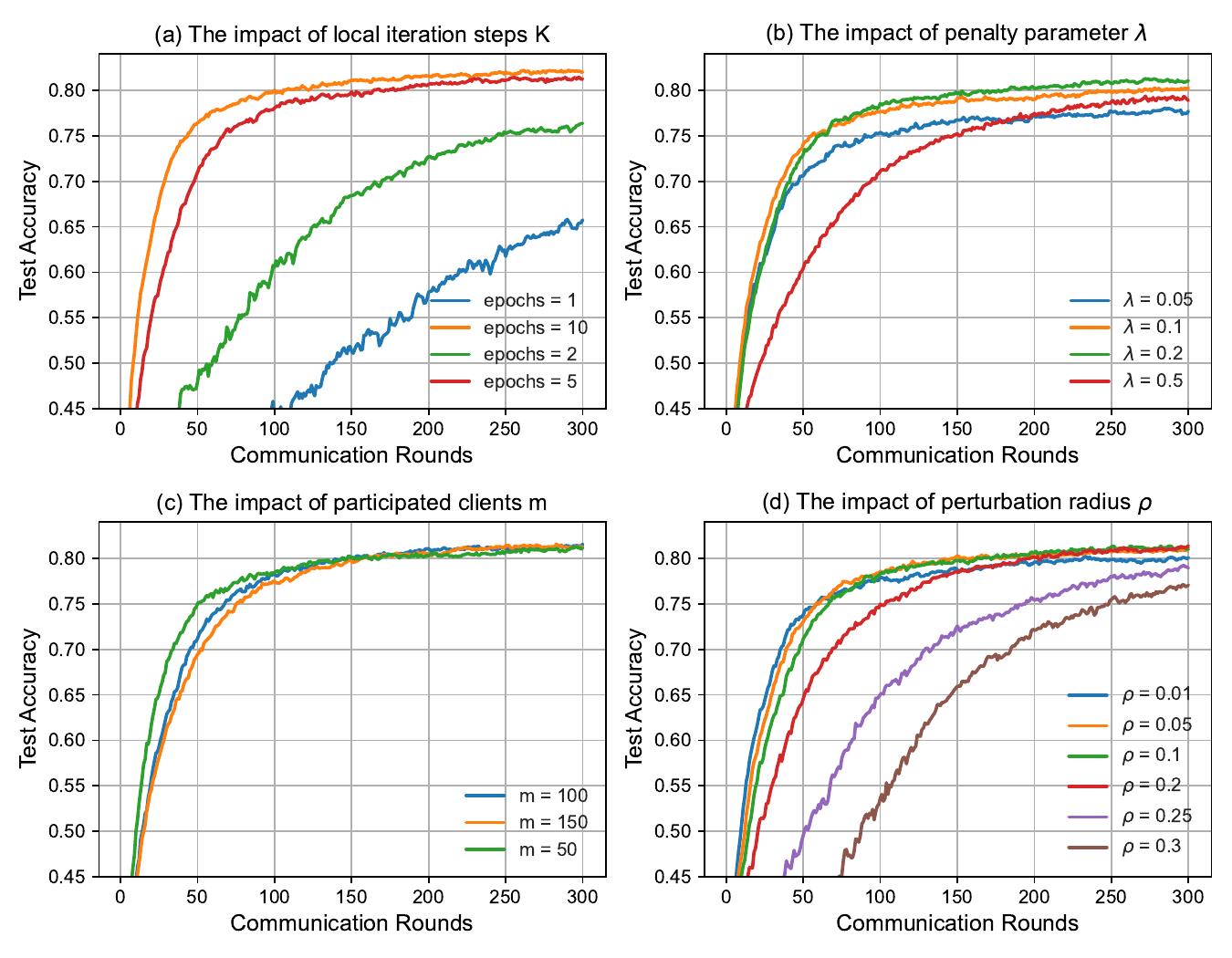}
    }
\end{center}
 \vspace{-0.4cm}
\caption{ \small Sensitivity of hyper-parameters: local iteration K, penalty parameter $\lambda$, number of participated clients m, perturbation radius $\rho$, respectively.}
\label{fig:hyper}
 \vspace{-0.3cm}
\end{figure}

\subsubsection{Number of participated clients $m$.}
In Figure \ref{fig:hyper} (c), we compare the performance across different numbers of client participation ($m=\{50, 100, 150\}$) while keeping the hyper-parameters consistent. The experimental results demonstrate the robustness of our algorithm with respect to the number of clients $m$. We observe that the convergence speed and generalization performance of the algorithm show no significant difference despite varying values of $m$. This is a crucial factor to consider when deploying the algorithm in practical scenarios, as in real-world federated learning settings, $m$ is often very large.

\subsubsection{Local iteration steps $K$.}
The previous work \cite{Sun2022Decentralized,shi2023improving} with theoretical guarantees has already indicated that larger values of $K$ can accelerate convergence. In order to explore the acceleration effect of larger local iteration steps $K$, we kept other hyper-parameters fixed and varied $K$ within the set $\{1, 2, 5, 10\}$. The results indicate that a larger value of $K=10$ exhibits faster convergence compared to $\{1, 2, 5\}$. This aligns with the results obtained in the convergence analysis section, which demonstrates that larger values of $K$ can tighten the upper bound.

\subsubsection{Perturbation radius $\rho$.} The perturbation radius $\rho$ also influences the performance as the accumulated perturbation increases with each communication round $T$. Generally, the choice of $\rho$ should not be too large, as it may obscure the true direction of gradients, leading to slow convergence or even divergence \cite{andriushchenko2022towards,zhu2023decentralized}.
To determine an appropriate value for DFedADMM-SAM, we conduct experiments using different perturbation radius from the set $\{0.01, 0.05, 0.1, 0.2, 0.25 ,0.3\}$, as depicted in Figure \ref{fig:hyper} (d). The results indicate that an appropriate value of $\rho$ can enhance the model's generalization ability, and within the given set, $\rho = 0.1$ is the optimal value. Furthermore, it is observed that as $\rho$ increases, the convergence speed gradually decreases, which aligns with our previous analysis.

\subsubsection{Penalty parameter \(\lambda\).} The penalty parameter $\lambda$ has an impact on the performance as the communication round $T$ increases. It represents a trade-off between the local optima and the consistency of the aggregated model. If $\lambda$ is too small, the client model parameters may get trapped near the aggregated model parameters, making it difficult to optimize. On the other hand, if $\lambda$ is too large, clients may move towards local optima, exacerbating local inconsistency. To select an appropriate value for the penalty parameter in our algorithms, we conduct several experiments using different values from the set $\{0.05, 0.1, 0.2, 0.5\}$, as shown in Figure \ref{fig:hyper} (b). We observe that we achieve better performance with $\lambda = 0.2$  compared to other values in the set.

\section{Conclusions And Future Work}
In this paper, we address two challenges in Decentralized Federated Learning (DFL): local inconsistency and local heterogeneous overfitting. We propose two DFL frameworks: DFedADMM and DFedADMM-SAM, to achieve better model consistency among clients.  DFedADMM addresses model inconsistency caused by heterogeneous data by introducing a dual constraint, and DFedADMM-SAM is an enhanced version of DFedADMM that incorporates gradient perturbation based on the Sharpness-Aware Minimization (SAM) optimizer to generate locally flat models. This helps alleviate local heterogeneous overfitting by seeking models with uniformly low loss values.
For theoretical findings, We explore the convergence rate in DFL by investigating the impacts of gradient perturbation in SAM, the local epoch K, the penalty parameter in ADMM, the network topology, and data homogeneity. Our theoretical findings provide a unified understanding of these factors on the convergence rate in DFL.
Furthermore, empirical results further validate the effectiveness of our proposed approaches. Moving forward, future work will focus on understanding the generalization of ADMM in the context of DFL.

	\bibliography{aaai24}

	\clearpage
	\appendix
	\onecolumn
	\vspace{0.5in}
	\begin{center}
		\rule{6.875in}{0.7pt}\\ 
		{\Large\bf Appendix for `` DFedADMM: Dual Constraints Controlled Model Inconsistency \\ for Decentralized Federated Learning ''}
		\rule{6.875in}{0.7pt}
	\end{center}
	\section{ More Details on Algorithm Implementation}\label{implemental_details}
\subsection{Datasets and backbones.} 
MNIST is a relatively simple dataset for handwritten digit recognition. CIFAR-10 and CIFAR-100 \cite{krizhevsky2009learning} are labeled subsets of the larger 80 million images dataset. These three datasets share the same set of 60,000 input images. CIFAR-100 provides a more detailed labeling with 100 unique labels, while CIFAR-10 has 10 unique labels. The ResNet as the backbone is used for CIFAR-100, where the batch-norm layers are replaced by group-norm layers due to a detrimental effect of batch-norm. The backbone architecture used for MNIST is as follows: Cascading three fully connected layers, the architecture includes 200 neurons in the initial layer, followed by an additional 200 neurons in the subsequent layer, and concludes with 10 output neurons in the final layer.This architecture can be found in the work of \cite{Sun2022Decentralized}. The backbone architecture used for CIFAR-10 is as follows: The architecture of the network consists of two convolutional layers, one max pooling layer, and three fully connected layers. The first convolutional layer takes a 3-channel input and generates a 64-channel output using a kernel size of 5. Similarly, the second convolutional layer has 64 input channels, 64 output channels, and a kernel size of 5. Following each convolutional layer, spatial dimensions are reduced through max pooling with a kernel size of 2 and a stride of 2. The subsequent fully connected layers comprise 384 neurons in the first layer, 192 neurons in the second layer, and a final layer with an output size that matches the number of classes in the classification task.

\subsection{More details about baselines.} 
FedAvg is a classical FL method that utilizes weighted averaging to train a global model in parallel with a central server. FedSAM applies the Sharpness-Aware Minimization (SAM) method as the local optimizer to enhance the model's generalization performance. In the decentralized setting, D-PSGD is a well-known decentralized parallel SGD method that aims to reach a consensus model. It should be noted that in this work, we focus on decentralized FL, which involves multiple local iterates during training, as opposed to decentralized learning/training that focuses on one-step local training. For instance, D-PSGD \cite{lian2017can} is a decentralized training algorithm that uses one-step SGD to train local models in each communication round. DFedAvg is a decentralized version of FedAvg, where clients perform local training and share updated models through communication rounds. DFedAvgM extends on DFedAvg by incorporating SGD with momentum, where each client performs multiple local training steps before each communication round. Furthermore, DFedSAM enhances the generalization ability of the aggregated model by applying gradient perturbation techniques to generate locally flattened models.

\subsection{Hyperparameters.}
In our experiments, we set the total number of clients to 100 for both centralized and decentralized settings. In the decentralized setting, each client is connected to at most 10 neighbors. For the centralized setting, the sample ratio of clients is set to 0.1.
The local learning rate is initialized to 0.1 and decayed by a factor of 0.998 after each communication round in all experiments. For centralized methods, the global learning rate is set to 1.0. The batch size is fixed at 128 for all experiments. We run 500 global communication rounds for CIFAR-10/100 and 300 global communication rounds for MNIST.
We use the SGD optimizer with a weighted decay parameter of 0.0005 for all baselines except for DFedADMM and DFedADMM-SAM. In algorithms based on the SAM optimizer, such as FedSAM, DFedSAM, and DFedADMM-SAM, the hyper-parameter for gradient perturbation, denoted as $\rho$, is set to 0.1. In addition, in our algorithm, the penalty parameter $\lambda$ is set to 0.1 for all experiments. Following the approach in \cite{Sun2022Decentralized}, we use a momentum of 0.9 for the local optimization in DFedAvgM. For the number of local iterations $K$, the training epoch in D-PSGD is set to 1, while for all other methods, it is set to 5.

\subsection{Communication configurations.} 
In the decentralized methods, such as mentioned in \cite{Rong2022DisPFL}, the communication volume is generally higher compared to centralized methods. This is because each client in the network topology needs to transmit their local information to their neighbors. However, in the centralized setting, only a subset of clients is sampled to upload their parameter updates to a central server. To ensure a fair comparison, we adopt a dynamic and time-varying connection topology for the decentralized methods in \textbf{Section} \ref{exper-evaluation}. Specifically, we restrict each client to communicate with a maximum of 10 randomly selected neighbors, without replacement, from all the clients. Moreover, only 10 clients, who are neighbors to each other, are allowed to perform one gossip step to exchange their local information in the decentralized methods.
\newpage
\section{More experiments results}
\subsection{Convergence speed}
We demonstrate the needed communication rounds for each method to reach a target accuracy as follows. Results for MNIST are in Table \ref{MNIST-convergence}, while results for CIFAR-10 in Table \ref{CIFAR10-convergence} and CIFAR1-100 in Table \ref{CIFAR100-convergence}.

\begin{table*}[ht]
\centering
\small
\vskip 0.15in
\resizebox{15cm}{!}{\begin{tabular}{cccc|ccc|ccc}
\toprule
\multirow{2}{*}{Methods} & \multicolumn{3}{c|}{Dir 0.3} & \multicolumn{3}{c|}{Dir 0.6} & \multicolumn{3}{c}{IID} \\
\cmidrule{2-10}
 & Acc@97      & Acc@98      & Acc@98.2      & Acc@97.5      & Acc@98      & Acc@98.2      & Acc@97.5      & Acc@98      & Acc@98.2    \\ \midrule
FedAvg & 60 & \textgreater 300 & \textgreater 300 & 76 & 194 & \textgreater 300 & 46 & 108 & 135 \\

FedSAM & 41 & 135 & 201 & 55 & 121 & 177 & 33 & 68 & 98 \\

D-PSGD & 223 & \textgreater 300 & \textgreater 300 & \textgreater 300 & \textgreater 300 & \textgreater 300 & 227 & \textgreater 300 & \textgreater 300 \\

DFedAvg & 87 & \textgreater 300 & \textgreater 300 & 107 & \textgreater 300 & \textgreater 300 & 64 & 229 & \textgreater 300 \\

DFedAvgM & 19 & 104 & \textgreater 300 & 26 & 67 & 258 & 19 & 53 & 283 \\

DFedSAM & 51 & 140 & 218 & 67 & 163 & 294 & 53 & 105 & 200 \\

DFedADMM & 19 & 115 & \textgreater 300 & 22 & 45 & \textgreater 300 & 18 & 53 & \textgreater 300 \\

DFedADMM-SAM & \textbf{16} & \textbf{32} & \textbf{40} & \textbf{21} & \textbf{30} & \textbf{41} & \textbf{16} & \textbf{25} & \textbf{32} \\
\midrule
\end{tabular}}
\captionsetup{position=bottom} 
\caption{Averaged needed communication rounds for each method to a target accuracy on MNIST dataset.}
\label{MNIST-convergence}
\vskip -0.1in
\end{table*}

\begin{table*}[h]
\centering
\small
\resizebox{15cm}{!}{\begin{tabular}{cccc|ccc|ccc}
\toprule
\multirow{2}{*}{Methods} & \multicolumn{3}{c|}{Dir 0.3} & \multicolumn{3}{c|}{Dir 0.6} & \multicolumn{3}{c}{IID} \\
\cmidrule{2-10}
 & Acc@73      & Acc@75      & Acc@80      & Acc@75         & Acc@79         & Acc@81      & Acc@78         & Acc@80         & Acc@82    \\ \midrule
FedAvg & 113 & 154 & \textgreater 500 & 121 & 392 & \textgreater 500 & 143 & 243 & \textgreater 500 \\

FedSAM & 120 & 141 & 392 & 111 & 211 & 399 & 128 & 195 & 329 \\

D-PSGD & 500 & \textgreater 500 & \textgreater 500 & \textgreater 500 & \textgreater 500 & \textgreater 500 & 499 & \textgreater 500 & \textgreater 500 \\

DFedAvg & 124 & 177 & \textgreater 500 & 138 & \textgreater 500 & \textgreater 500 & 190 & \textgreater 500 & \textgreater 500 \\

DFedAvgM & 61 & 91 & \textgreater 500 & 61 & 217 & \textgreater 500 & 58 & 101 & 237 \\

DFedSAM & 145 & 182 & 473 & 146 & 281 & \textgreater 500 & \textgreater 237 & \textgreater 237 & \textgreater 237 \\

DFedADMM & \textbf{45} & \textbf{62} & 330 & \textbf{52} & 183 & \textgreater 500 & 60 & 236 & \textgreater 500 \\

DFedADMM-SAM & 63 & 72 & \textbf{167} & 53 & \textbf{95} & \textbf{165} & \textbf{50} & \textbf{67} & \textbf{125} \\
\midrule
\end{tabular}}
\captionsetup{position=bottom} 
\caption{Averaged needed communication rounds for each method to a target accuracy on CIFAR-10 dataset.}
\label{CIFAR10-convergence}
\end{table*}

\begin{table*}[h]
\centering
\small
\resizebox{15cm}{!}{\begin{tabular}{cccc|ccc|ccc}
\toprule
\multirow{2}{*}{Methods} & \multicolumn{3}{c|}{Dir 0.3} & \multicolumn{3}{c|}{Dir 0.6} & \multicolumn{3}{c}{IID} \\
\cmidrule{2-10}
 & Acc@50      & Acc@55      & Acc@58      & Acc@52         & Acc@55         & Acc@58      & Acc@55         & Acc@58         & Acc@60    \\ \midrule
FedAvg & 98 & 381 & \textgreater 500 & 105 & 257 & \textgreater 500 & 126 & \textgreater 500 & \textgreater 500 \\

FedSAM & 136 & 311 & \textgreater 500 & 147 & 261 & \textgreater 500 & \textgreater 500 & \textgreater 500 & \textgreater 500 \\

D-PSGD & 232 & 402 & \textgreater 500 & 261 & 386 & \textgreater 500 & 320 & \textgreater 500 & \textgreater 500 \\

DFedAvg & 57 & 111 & 268 & 61 & 95 & 212 & 66 & 135 & 391 \\

DFedAvgM & 59 & \textgreater 500 & \textgreater 500 & 47 & 111 & \textgreater 500 & 65 & \textgreater 500 & \textgreater 500 \\

DFedSAM & 91 & 153 & 287 & 92 & 135 & 233 & 86 & 141 & 318 \\

DFedADMM & \textbf{42} & \textbf{67} & 123 & \textbf{33} & \textbf{42} & \textbf{72} & \textbf{28} & \textbf{40} & \textbf{54} \\

DFedADMM-SAM & 44 & \textbf{67} & \textbf{108} & 40 & 54 & 76 & 31 & 45 & 70 \\
\midrule
\end{tabular}}
\captionsetup{position=bottom} 
\caption{Averaged needed communication rounds for each method to a target accuracy on CIFAR-100 dataset.}
\label{CIFAR100-convergence}
\end{table*}

From Tables \ref{MNIST-convergence}, \ref{CIFAR10-convergence}, and \ref{CIFAR100-convergence}, it can be observed that our proposed algorithms, DFedADMM and DFedADMM-SAM, achieve the minimum communication cost for a given target accuracy. Compared to other baseline methods, our algorithms are able to reach 90\% of the final convergence accuracy in fewer communication rounds. This result demonstrates the fast convergence speed of our algorithms.

\newpage

\subsection{Different topology}

\begin{figure*}[ht]
\begin{center}
\subfigure{
    	\includegraphics[width=0.7\textwidth]{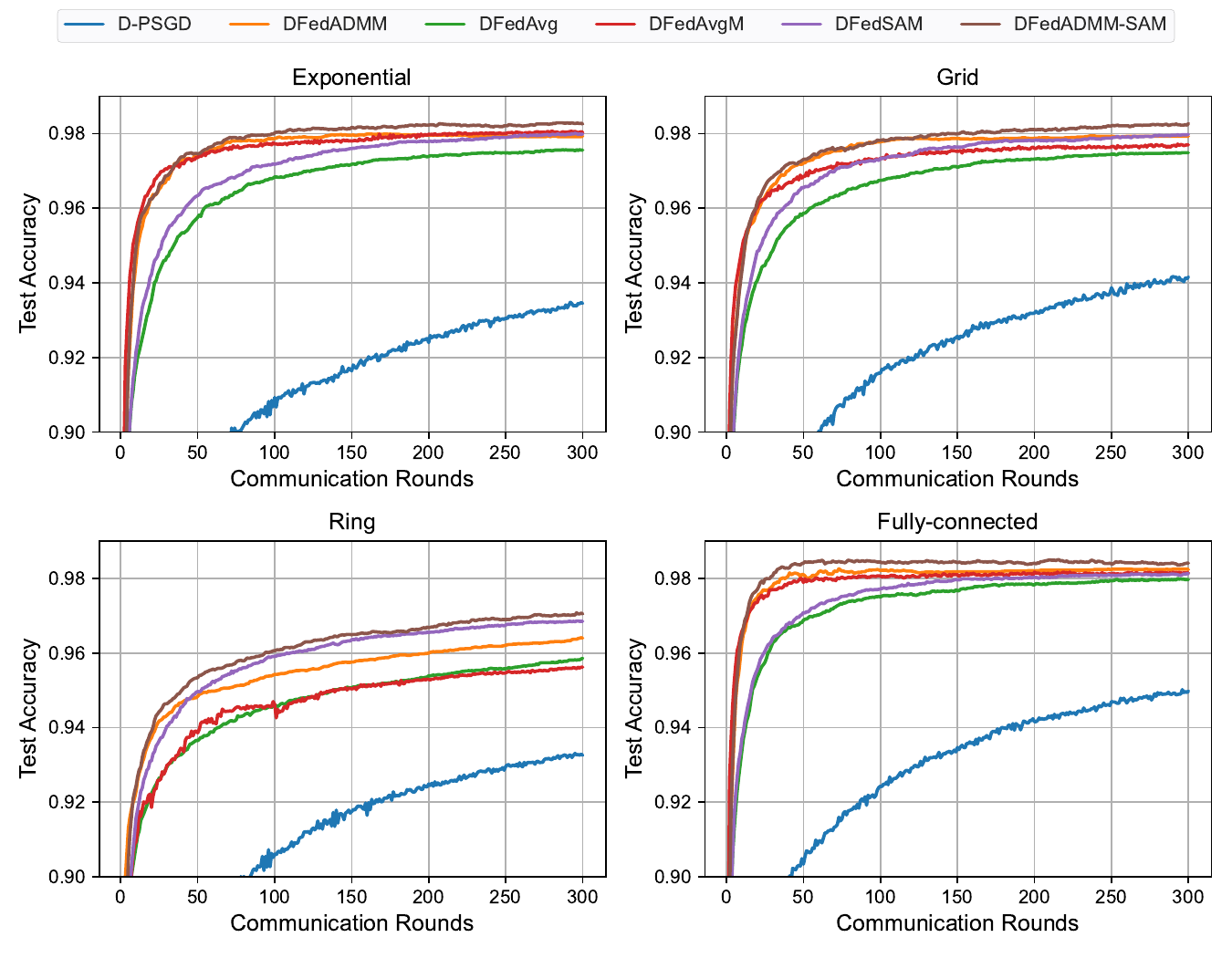}
    }
\end{center}
 \vspace{-0.4cm}
\caption{ \small Accuracy of different DFL algorithms with different decentralized topologies on the test dataset.}
\label{fig:topo_aware}
 \vspace{-0.3cm}
\end{figure*}

This Figure \ref{fig:topo_aware} complements Table \ref{ta:topo}, showing that as the connectivity of the communication topology improves, the proposed algorithm achieves higher accuracy on the test set. This observation aligns with the conclusion of Corollary \ref{coro_DFedADMM} in Section \ref{analysis}.

\clearpage

\subsection{Performance Evaluation}

\begin{figure*}[h]
\begin{center}
\subfigure[MNIST]{
    	\includegraphics[width=0.9\textwidth]{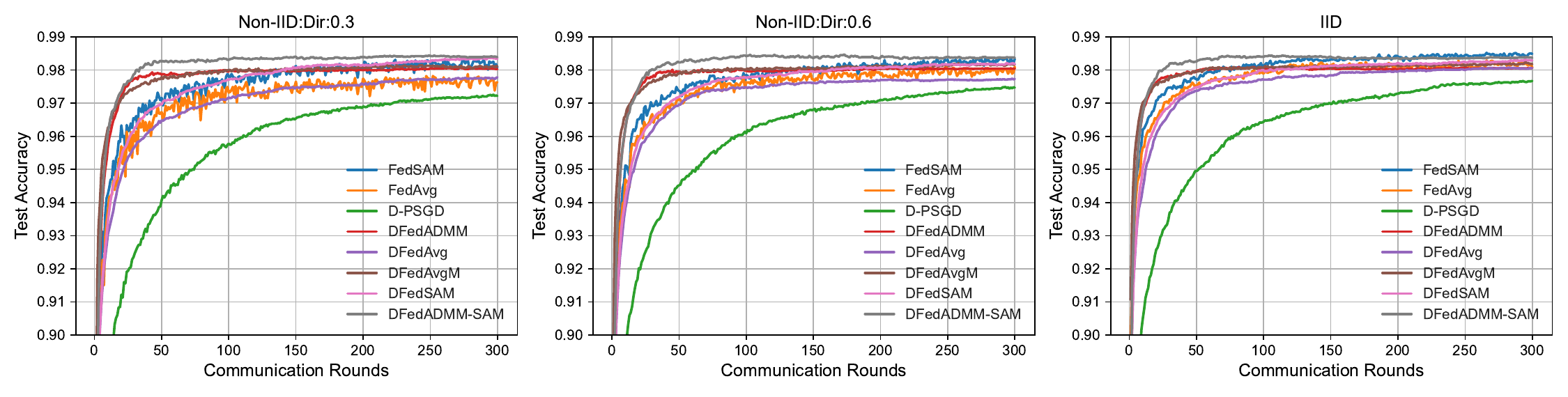}
        \label{mnist_acc}
    }
\subfigure[CIFAR-10]{
    	\includegraphics[width=0.9\textwidth]{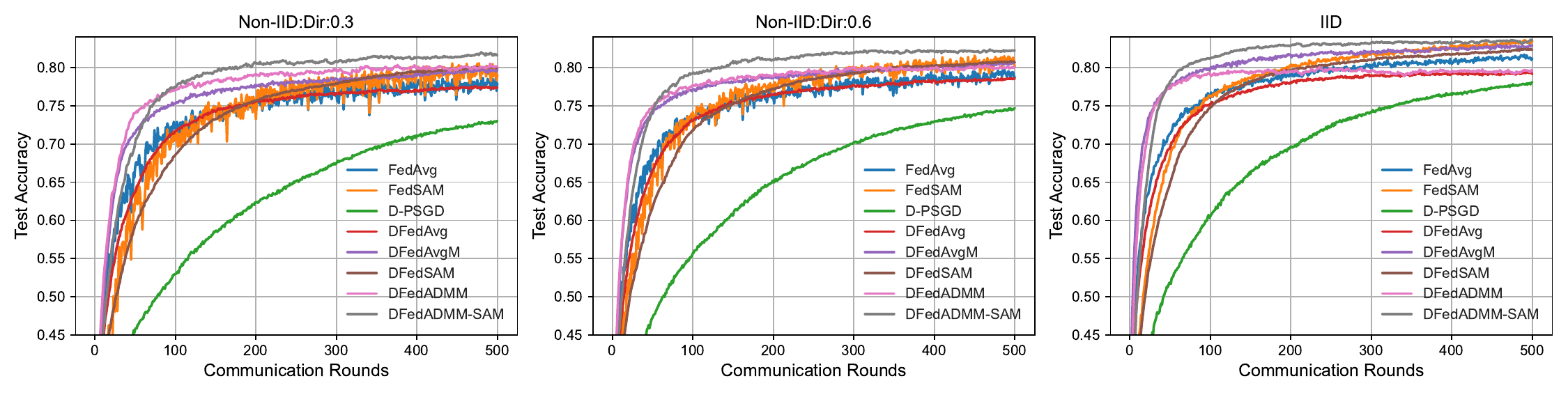}
        \label{cifar10_acc}
    }
\subfigure[CIFAR-100]{
    \includegraphics[width=0.9\textwidth]{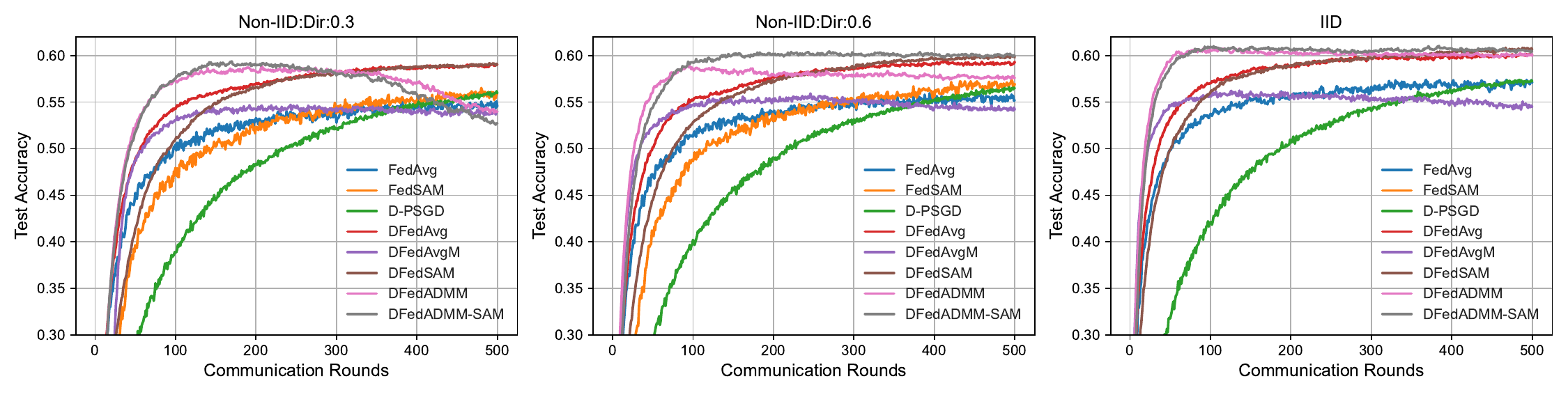}
    \label{cifar100_acc}
}
\end{center}
\caption{ \small Test accuracy of all baselines on (a) MNIST, (b) CIFAR-10, and (c) CIFAR-100 in both IID and non-IID settings.}
\label{fig:Compared_baselines}
\vspace{-0.3cm}
\end{figure*}
Similar to Table \ref{ta:all_baselines}, Figure \ref{fig:Compared_baselines} presents the performance of each method on different datasets and with different data partitioning schemes.



\clearpage
	
\newpage

\section{Proofs for Analysis}\label{appendix_proof}
	In this part we will demonstrate the proofs of all formula mentioned in this paper. Each formula is presented in the form of a lemma.
	
	\begin{lemma}\label{mi}[Lemma 4, \cite{lian2017can}]
		For any $t\in \mathbb{Z}^+$, the mixing matrix ${\bf W}\in\ \mathbb{R}^m$ satisfies
		$\|{\bf W}^t-{\bf P}\|_{\emph{op}}\leq \psi^t,$
		where $\psi:=\max\{|\psi_2(\bf W)|,|\psi_m(\bf W)|\}$ and for a matrix ${\bf A}$, we denote its spectral norm as $\|{\bf A}\|_{\emph{op}}$. Furthermore, ${\bf 1}:=[1, 1, \ldots, 1 ]^{\top}\in \mathbb{R}^m$ and
		\begin{equation*}
			{\bf P}:=\frac{\mathbf{1}\mathbf{1}^{\top}}{m}\in \mathbb{R}^{m\times m}.
		\end{equation*}
	\end{lemma}
	In [Proposition 1, \cite{nedic2009distributed}], the author also proved that $\|{\bf W}^t-{\bf P}\|_{\textrm{op}}\leq C\psi^t$ for some $C>0$ that depends on the matrix.
	
	\begin{lemma}\label{Delta}
		For $\forall \ \mathbf{x}_{i,k}^{t} \in \mathbb{R}^{d}$ and $i \in \mathcal{S}^{t}$, we denote $\delta_{i,k}^{t}=\mathbf{x}_{i,k}^{t}-\mathbf{x}_{i,k-1}^{t}$ with setting $\delta_{i,0}^{t}=0$, and $\Delta_{i,K}^{t}=\sum_{k=0}^{K}\delta_{i,k}^{t} = \mathbf{x}_{i,K}^{t}-\mathbf{x}_{i,0}^{t}$, under the update rule in Algorithm Algorithm 1, we have:
		\begin{equation}
			\Delta_{i,K}^{t} =-\lambda\gamma\sum_{k=0}^{K-1}\frac{\gamma_{k}}{\gamma}\mathbf{g}_{i,k}^{t} + \gamma\lambda\hat{\mathbf{g}}_{i}^{t-1}
		\end{equation}
		where $\sum_{k=0}^{K-1}\gamma_{k}=\sum_{k=0}^{K-1}\frac{\eta_{l}}{\lambda}\bigl(1-\frac{\eta_{l}}{\lambda}\bigr)^{K-1-k}=\gamma=1-(1-\frac{\eta_{l}}{\lambda})^{K}$.
		\begin{proof}
			According to the update rule of Line.11 in Algorithm 1, we have:
			\begin{align*}
				\delta_{k} 
				&= \Delta_{i,k}^{t} - \Delta_{i,k-1}^{t}= \mathbf{x}_{i,k}^{t}-\mathbf{x}_{i,k-1}^{t}\\
				&= -\eta_{l}\bigl(\mathbf{g}_{i,k-1}^{t} - \hat{\mathbf{g}}_{i}^{t-1} + \frac{1}{\lambda}(\mathbf{x}_{i,k-1}^{t}-\mathbf{x}_{i,0}^{t})\bigr) = -\eta_{l}(\mathbf{g}_{i,k-1}^{t} - \hat{\mathbf{g}}_{i}^{t-1} + \frac{1}{\lambda}\Delta_{i,k-1}^{t}).
			\end{align*}
			Then We can formulate the iterative relationship of $\Delta_{i,k}^{t}$ as:
			\begin{align*}
				\Delta_{i,k}^{t} = \Delta_{i,k-1}^{t} -\eta_{l}(\mathbf{g}_{i,k-1}^{t} - \hat{\mathbf{g}}_{i}^{t-1} + \frac{1}{\lambda}\Delta_{i,k-1}^{t})=(1-\frac{\eta_{l}}{\lambda})\Delta_{i,k-1}^{t} - \eta_{l}(\mathbf{g}_{i,k-1}^{t}-\hat{\mathbf{g}}_{i}^{t-1}).
			\end{align*}
			Taking the iteration on $k$ and we have:
			\begin{align*}
				\mathbf{x}_{i,K}^{t}-\mathbf{x}_{i,0}^{t} = \Delta_{i,K}^{t}
				&= (1-\frac{\eta_{l}}{\lambda})^{K}\Delta_{i,0}^{t} -  \eta_{l}\sum_{k=0}^{K-1}(1-\frac{\eta_{l}}{\lambda})^{K-1-k}(\mathbf{g}_{i,k}^{t} - \hat{\mathbf{g}}_{i}^{t-1})\\
				&\overset{(a)}{=}  -  \eta_{l}\sum_{k=0}^{K-1}(1-\frac{\eta_{l}}{\lambda})^{K-1-k}(\mathbf{g}_{i,k}^{t} - \hat{\mathbf{g}}_{i}^{t-1})\\
				&= -\lambda\sum_{k=0}^{K-1}\frac{\eta_{l}}{\lambda}(1-\frac{\eta_{l}}{\lambda})^{K-1-k}(\mathbf{g}_{i,k}^{t} - \hat{\mathbf{g}}_{i}^{t-1})\\
				&= -\lambda\sum_{k=0}^{K-1}\frac{\eta_{l}}{\lambda}(1-\frac{\eta_{l}}{\lambda})^{K-1-k}\mathbf{g}_{i,k}^{t} + \bigl(1-(1-\frac{\eta_{l}}{\lambda})^{K}\bigr)\lambda\hat{\mathbf{g}}_{i}^{t-1}\\
				&= -\lambda\gamma\sum_{k=0}^{K-1}\frac{\gamma_{k}}{\gamma}\mathbf{g}_{i,k}^{t} + \gamma\lambda\hat{\mathbf{g}}_{i}^{t-1}.
			\end{align*}
			(a) applies $\Delta_{i,0}^{t}=\delta_{i,0}^{t}=0$.\\
		\end{proof}
	\end{lemma}

	\begin{lemma}\label{update_g_hat}
		Under the update rule in Algorithm Algorithm 1, we have:
		\begin{equation}
			\hat{\mathbf{g}}_{i}^{t} = (1-\gamma)\hat{\mathbf{g}}_{i}^{t-1} + \gamma\sum_{k=0}^{K-1}\frac{\gamma_{k}}{\gamma}\mathbf{g}_{i,k}^{t}.
		\end{equation}
		where $\sum_{k=0}^{K-1}\gamma_{k}=\sum_{k=0}^{K-1}\frac{\eta_{l}}{\lambda}\bigl(1-\frac{\eta_{l}}{\lambda}\bigr)^{K-1-k}=\gamma=1-(1-\frac{\eta_{l}}{\lambda})^{K}$.
		\begin{proof}
			According to the update rule of Line.13 in Algorithm 1, we have:
			\begin{align*}
				\hat{\mathbf{g}}_{i}^{t}
				&= \hat{\mathbf{g}}_{i}^{t-1}-\frac{1}{\lambda} (\mathbf{x}_{i,K}^{t}-\mathbf{x}_{i,0}^{t})\\
				&\overset{(a)}{=} \hat{\mathbf{g}}_{i}^{t-1}+\frac{\eta_{l}}{\lambda}\sum_{k=0}^{K-1}\bigl(1-\frac{\eta_{l}}{\lambda}\bigr)^{K-1-k}(\mathbf{g}_{i,k}^{t}-\hat{\mathbf{g}}_{i}^{t-1})\\
				&= \hat{\mathbf{g}}_{i}^{t-1} + \frac{\eta_{l}}{\lambda}\sum_{k=0}^{K-1}\bigl(1-\frac{\eta_{l}}{\lambda}\bigr)^{K-1-k}\mathbf{g}_{i,k}^{t} - \frac{\eta_{l}}{\lambda}\Bigl(\sum_{k=0}^{K-1}\bigl(1-\frac{\eta_{l}}{\lambda}\bigr)^{K-1-k}\Bigr)\hat{\mathbf{g}}_{i}^{t-1}\\
				&= \hat{\mathbf{g}}_{i}^{t-1} + \frac{\eta_{l}}{\lambda}\sum_{k=0}^{K-1}\bigl(1-\frac{\eta_{l}}{\lambda}\bigr)^{K-1-k}\mathbf{g}_{i,k}^{t} - \frac{\eta_{l}}{\lambda}\frac{1-(1-\frac{\eta_{l}}{\lambda})^{K}}{\frac{\eta_{l}}{\lambda}}\hat{\mathbf{g}}_{i}^{t-1}\\
				&= (1-\frac{\eta_{l}}{\lambda})^{K}\hat{\mathbf{g}}_{i}^{t-1} + \frac{\eta_{l}}{\lambda}\sum_{k=0}^{K-1}\bigl(1-\frac{\eta_{l}}{\lambda}\bigr)^{K-1-k}\mathbf{g}_{i,k}^{t}\\
				&= (1-\gamma)\hat{\mathbf{g}}_{i}^{t-1} + \gamma\sum_{k=0}^{K-1}\frac{\gamma_{k}}{\gamma}\mathbf{g}_{i,k}^{t}\\
			\end{align*}
			(a) applies the Lemma \ref{Delta}.
		\end{proof}
	\end{lemma}

	\begin{lemma}\label{true_global_update}
		(Bounded global update) The global update $\frac{1}{m}\sum_{i\in[m]}\hat{\mathbf{g}}_{i}^{t}$ holds the upper bound of:
		\begin{align*}
			\mathbb{E}_{t}\Vert\frac{1}{m}\sum_{i\in[m]}\hat{\mathbf{g}}_{i}^{t}\Vert^{2}\leq
			\sigma_l^2 + B^2
		\end{align*}
	\end{lemma}
	\begin{proof}
		According to the lemma~\ref{update_g_hat},we have:
		\begin{align*}
			\frac{1}{m}\sum_{i\in[m]}\hat{\mathbf{g}}_{i}^{t} = (1-\gamma)\frac{1}{m}\sum_{i\in[m]}\hat{\mathbf{g}}_{i}^{t-1} + \gamma\frac{1}{m}\sum_{i\in[m]}\sum_{k=0}^{K-1}\frac{\gamma_{k}}{\gamma}\mathbf{g}_{i,k}^{t}.\\
		\end{align*}
		Take the L2-norm and we have:
		\begin{align*}
			\Vert\frac{1}{m}\sum_{i\in[m]}\hat{\mathbf{g}}_{i}^{t}\Vert^{2}
			&= \Vert(1-\gamma)\frac{1}{m}\sum_{i\in[m]}\hat{\mathbf{g}}_{i}^{t-1} + \gamma\frac{1}{m}\sum_{i\in[m]}\sum_{k=0}^{K-1}\frac{\gamma_{k}}{\gamma}\mathbf{g}_{i,k}^{t}\Vert^{2}\\
			&\leq (1-\gamma)\Vert\frac{1}{m}\sum_{i\in[m]}\hat{\mathbf{g}}_{i}^{t-1}\Vert^{2} + \gamma\Vert\frac{1}{m}\sum_{i\in[m]}\sum_{k=0}^{K-1}\frac{\gamma_{k}}{\gamma}\mathbf{g}_{i,k}^{t}\Vert^{2}.\\
			&\leq
			(1-\gamma)\Vert\frac{1}{m}\sum_{i\in[m]}\hat{\mathbf{g}}_{i}^{t-1}\Vert^{2} + \gamma\left( \sigma_{l}^2 + B^2 \right)\\
		\end{align*}
		Thus we have the following recursion,
		\begin{align*}
			\mathbb{E}_{t}\Vert\frac{1}{m}\sum_{i\in[m]}\hat{\mathbf{g}}_{i}^{t-1}\Vert^{2}\leq (1-\gamma)\mathbb{E}_t\Vert\frac{1}{m}\sum_{i\in[m]}\hat{\mathbf{g}}_{i}^{t-1}\Vert^{2} + \gamma\left( \sigma_{l}^2 + B^2 \right)\leq \left(\sigma_{l}^2 + B^2\right)
		\end{align*}
		Where we use the fact $0< \gamma < 1$
	\end{proof}
	
	\begin{lemma}\label{Bounded local update}
		(Bounded local update) The local update $\hat{\mathbf{g}}_{i}^{t}$ holds the upper bound of:
		\begin{align*}
			\frac{1}{m}\sum_{i\in[m]}\mathbb{E}_{t}\Vert\hat{\mathbf{g}}_{i}^{t}\Vert^{2}
			\leq (\sigma_{l}^2 + B^2)
		\end{align*}
		
		\begin{proof}
			According to the lemma\ref{update_g_hat}, we have:
			\begin{align*}
				\hat{\mathbf{g}}_{i}^{t} = (1-\gamma)\hat{\mathbf{g}}_{i}^{t-1} + \gamma\sum_{k=0}^{K-1}\frac{\gamma_{k}}{\gamma}\mathbf{g}_{i,k}^{t}.
			\end{align*}
			Take the L2-norm and we have:
			\begin{align*}
				\Vert\hat{\mathbf{g}}_{i}^{t}\Vert^{2}
				&= \Vert(1-\gamma)\hat{\mathbf{g}}_{i}^{t-1} + \gamma\sum_{k=0}^{K-1}\frac{\gamma_{k}}{\gamma}\mathbf{g}_{i,k}^{t}\Vert^{2}\\
				&\overset{(a)}{\leq} (1-\gamma)\Vert\hat{\mathbf{g}}_{i}^{t-1}\Vert^{2} + \gamma\Vert\sum_{k=0}^{K-1}\frac{\gamma_{k}}{\gamma}\mathbf{g}_{i,k}^{t}\Vert^{2}\\
				&\overset{(b)}{\leq} (1-\gamma)\Vert\hat{\mathbf{g}}_{i}^{t-1}\Vert^{2} + \gamma\sum_{k=0}^{K-1}\frac{\gamma_{k}}{\gamma}\Vert\mathbf{g}_{i,k}^{t}\Vert^{2}
			\end{align*}
			(a) and (b) hold by applying the Jensen inequality. Thus we have the following recursion:
			\begin{align*}
				\frac{1}{m}\sum_{i\in[m]}\mathbb{E}_{t}\Vert\hat{\mathbf{g}}_{i}^{t}\Vert^{2}&\leq
				(1-\gamma)\frac{1}{m}\sum_{i\in[m]}\mathbb{E}_{t}\Vert\hat{\mathbf{g}}_{i}^{t-1}\Vert^{2} + \gamma \frac{1}{m}\sum_{i\in[m]}\sum_{k=0}^{K-1}\frac{\gamma_k}{\gamma}\mathbb{E}_{t}\Vert\mathbf{g}_{i,k}^{t}\Vert^{2}\\
				&\leq (1-\gamma)\frac{1}{m}\sum_{i\in[m]}\mathbb{E}_{t}\Vert\hat{\mathbf{g}}_{i}^{t-1}\Vert^{2} + \gamma\left(\sigma_{l}^2 + B^2\right)\\
				&\leq \sigma_{l}^2 + B^2 \\       
			\end{align*}
		\end{proof}
	\end{lemma}

	\begin{lemma}\label{Bounded global error}
		(Bounded global update) Given the stepsize $ 0<\eta_{l} < \frac{1}{6LK}$, we have following bound:
		\begin{equation}
			\frac{1}{m}\sum_{i\in{m}}\sum_{k=0}^{K-1}\frac{\gamma_{k}}{\gamma}\mathbb{E}_{t}\Vert\mathbf{x}_{i,k}^{t}-\mathbf{x}_i^{t}\Vert^{2}\leq 16\eta_{l}^{2}K^2C_0
		\end{equation}
		where $C_0 = \left(3\sigma_{g}^{2}+\sigma_{l}^{2}\right) + 3\frac{1}{m}\sum_{i\in[m]}\mathbb{E}_{t}\Vert\nabla f(\mathbf{x}_{i}^{t})\Vert^2 + (\sigma_{l}^2 + B^2)$
		\begin{proof}
			We denote $\mathbf{c}^{t}=\frac{1}{m}\sum_{i\in{m}}\sum_{k=0}^{K-1}(\gamma_{k}/\gamma)\mathbb{E}_{t}\Vert\mathbf{x}_{i,k}^{t}-\mathbf{x}_i^{t}\Vert^{2}$ term as the local offset after $k$ iterations updates, we firstly consider the $\mathbf{c}_{k}^{t}=\frac{1}{m}\sum_{i\in{m}}\mathbb{E}_{t}\Vert\mathbf{x}_{i,k}^{t}-\mathbf{x}_i^{t}\Vert^{2}$ and it can be bounded as:
			\begin{align*}\mathbf{c}_{k}^{t} 
				&=\frac{1}{m}\sum_{i\in[m]}\mathbb{E}_{t}\Vert\mathbf{x}_{i,k}^{t}-\mathbf{x}_{i}^{t}\Vert^2=\frac{1}{m}\sum_{i\in[m]}\mathbb{E}_{t}\Vert\mathbf{x}_{i,k}^{t}-\mathbf{x}_{i,k-1}^{t}+\mathbf{x}_{i,k-1}^{t}-\mathbf{x}_{i,0}^{t}\Vert^2\\  
				&=\frac{1}{m}\sum_{i\in[m]}\mathbb{E}_{t}\Vert-\eta_{l}(\mathbf{g}_{i,k-1}^{t}-\hat{\mathbf{g}}_{i}^{t-1})+(1-\frac{\eta_{l}}{\lambda})(\mathbf{x}_{i,k-1}^{t}-\mathbf{x}_{i,0}^{t})\Vert^2\\  
				&\leq(1+a)(1-\frac{\eta_{l}}{\lambda})^2\frac{1}{m}\sum_{i\in[m]}\mathbb{E}_{t}\Vert\mathbf{x}_{i,k-1}^{t}-\mathbf{x}_{i,0}^{t}\Vert^2+(1+\frac{1}{a})\frac{\eta_{l}^{2}}{m}\sum_{i\in[m]}\mathbb{E}_{t}\Vert\mathbf{g}_{i,k-1}^{t}-\hat{\mathbf{g}}_{i}^{t-1}\Vert^2\\  
				&=(1+a)(1-\frac{\eta_{l}}{\lambda})^2\mathbf{c}_{k-1}^{t}+(1+\frac{1}{a})\frac{\eta_{l}^{2}}{m}\sum_{i\in[m]}\mathbb{E}_{t}\Vert\mathbf{g}_{i,k-1}^{t}-\hat{\mathbf{g}}_{i}^{t-1}\Vert^2\\  
				&=(1+\frac{1}{a})\frac{\eta_{l}^{2}}{m}\sum_{i\in[m]}\mathbb{E}_{t}\Vert\nabla f_{i}(\mathbf{x}_{i,k-1}^{t})-\hat{\mathbf{g}}_{i}^{t-1}\Vert^2  
				+(1+\frac{1}{a})\eta_{l}^2\sigma_{l}^2+(1+a)(1-\frac{\eta_{l}}{\lambda})^2\mathbf{c}_{k-1}^{t}\\  
				&\leq(1+\frac{1}{a})\frac{2\eta_{l}^2}{m}\sum_{i\in[m]}\left(\mathbb{E}_{t}\Vert\nabla f_{i}(\mathbf{x}_{i,k-1}^{t})\Vert^2+\mathbb{E}_{t}\Vert\hat{\mathbf{g}}_{i}^{t-1}\Vert^2\right)+(1+\frac{1}{a})\eta_{l}^2\sigma_{l}^2+(1+a)(1-\frac{\eta_{l}}{\lambda})^2\mathbf{c}_{k-1}^{t}\\  
				&\leq(1+\frac{1}{a})\frac{2\eta_{l}^2}{m}\sum_{i\in[m]}\mathbb{E}_{t}\Vert\nabla f_{i}(\mathbf{x}_{i,k-1}^{t})\Vert^2+(1+\frac{1}{a})\frac{2\eta_{l}^2}{m}\sum_{i\in[m]}\mathbb{E}_{t}\Vert\hat{\mathbf{g}}_{i}^{t-1}\Vert^2+(1+\frac{1}{a})\eta_{l}^2\sigma_{l}^2\\  
				& \quad+(1+a)(1-\frac{\eta_{l}}{\lambda})^2\mathbf{c}_{k-1}^{t}\\  & \leq(1+\frac{1}{a})\frac{2\eta_{l}^2}{m}\sum_{i\in[m]}\mathbb{E}_{t}\Vert\nabla f_{i}(\mathbf{x}_{i,k-1}^{t})-\nabla f_{i}(\mathbf{x}_{i}^{t})+\nabla f_{i}(\mathbf{x}_{i}^{t})-\nabla f_{}(\mathbf{x}_{i}^{t})+\nabla f_{}(\mathbf{x}_{i}^{t})\Vert^2 \\
				&\quad+(1+\frac{1}{a})2\eta_{l}^2\left( \sigma_{l}^2 +B^2 \right)+(1+\frac{1}{a})\eta_{l}^2\sigma_{l}^2+(1+a)(1-\frac{\eta_{l}}{\lambda})^2\mathbf{c}_{k-1}^{t}\\  
				&\leq(1+\frac{1}{a})\frac{6\eta_{l}^2L^2}{m}\sum_{i\in[m]}\mathbb{E}_{t}\Vert\mathbf{x}_{i,k-1}^{t}-\mathbf{x}_i^{t}\Vert^2+(1+\frac{1}{a})\eta_{l}^2(6\sigma_{g}^2+\sigma_{l}^2)+(1+\frac{1}{a})6\eta_{l}^2\frac{1}{m}\sum_{i\in[m]}\mathbb{E}_{t}\Vert\nabla f_{}(\mathbf{x}_{i}^{t})\Vert^2\\  
				&\quad+(1+\frac{1}{a})2\eta_{l}^2\left( \sigma_{l}^2 +B^2 \right) + (1+a)(1-\frac{\eta_{l}}{\lambda})^2\mathbf{c}_{k-1}^{t}\\  
				&\leq\left[(1+a)(1-\frac{\eta_{l}}{\lambda})^2+(1+\frac{1}{a})6\eta_{l}^2L^2\right]\mathbf{c}_{k-1}^{t}+(1+\frac{1}{a})\eta_{l}^2(6\sigma_{g}^2+\sigma_{l}^2)\\  
				& \quad+(1+\frac{1}{a})6\eta_{l}^2\frac{1}{m}\sum_{i\in[m]}\mathbb{E}_{t}\Vert\nabla f_{}(\mathbf{x}_{i}^{t})\Vert^2+(1+\frac{1}{a})2\eta_{l}^2\left( \sigma_{l}^2 +B^2 \right)\\ 
			\end{align*}
			we have:
			\begin{align}\label{c_k^t} 
				\begin{aligned}
					\mathbf{c}_{k}^{t} 
					&\leq\left[(1+a)(1-\frac{\eta_{l}}{\lambda})^2+(1+\frac{1}{a})6\eta_{l}^2L^2\right]\mathbf{c}_{k-1}^{t}+(1+\frac{1}{a})\eta_{l}^2(6\sigma_{g}^2+\sigma_{l}^2)\\  
					& \quad+(1+\frac{1}{a})6\eta_{l}^2\frac{1}{m}\sum_{i\in[m]}\mathbb{E}_{t}\Vert\nabla f_{}(\mathbf{x}_{i}^{t})\Vert^2+(1+\frac{1}{a})2\eta_{l}^2\left( \sigma_{l}^2 +B^2 \right)\\ 
				\end{aligned}
			\end{align}
			Set $a = 2K-1, \eta_{l}\leq 2\lambda$ and $ 0<\eta_{l} < \frac{1}{6LK}$, we get:
			\begin{align*}
				(1+a)(1-\frac{\eta_{l}}{\lambda})^2+(1+\frac{1}{a})6\eta_{l}^2L^2 < 1+\frac{1}{K-1}\\
			\end{align*}
			then we have:
			\begin{align}\label{c_k^t}
				\mathbf{c}_k^{t}
				\nonumber
				&\leq 4K\eta_{l}^2\sum_{j=0}^{K-1}\left(1+\frac{1}{K-1}\right)^j\Bigl(\left(3\sigma_{g}^{2}+\sigma_{l}^{2}\right) + 3\frac{1}{m}\sum_{i\in[m]}\mathbb{E}_{t}\Vert\nabla f(\mathbf{x}_{i}^{t})\Vert^2 + (\sigma_{l}^2 + B^2)\Bigr)\\
				\nonumber
				&=4K\eta_{l}^2\left(K-1\right)\left[\left(1+\frac{1}{K-1}\right)^K - 1\right]\Bigl(\left(3\sigma_{g}^{2}+\sigma_{l}^{2}\right) + 3\frac{1}{m}\sum_{i\in[m]}\mathbb{E}_{t}\Vert\nabla f(\mathbf{x}_{i}^{t})\Vert^2 + (\sigma_{l}^2 + B^2)\Bigr)\\
				&\overset{(a)}\leq 16K^2\eta_{l}^2\Bigl(\left(3\sigma_{g}^{2}+\sigma_{l}^{2}\right) + 3\frac{1}{m}\sum_{i\in[m]}\mathbb{E}_{t}\Vert\nabla f(\mathbf{x}_{i}^{t})\Vert^2 + (\sigma_{l}^2 + B^2)\Bigr)
			\end{align}
			Where (a) uses the inequality $ \left(1+\frac{1}{K-1}\right)^K<5 $ holds for any $K \geq 1$.
			
			Since the expression of $\mathbf{c}^t = \sum_{k=0}^{K-1} \frac{\gamma_k}{\gamma}\mathbf{c}_k^t$, the upper bound of $ \mathbf{c}^t $ is the same as the upper bound of $ \mathbf{c}_k^t $:
			\begin{align*}
				\mathbf{c}^{t}
				&\leq 16K^2\eta_{l}^2\Bigl(\left(3\sigma_{g}^{2}+\sigma_{l}^{2}\right) + 3\frac{1}{m}\sum_{i\in[m]}\mathbb{E}_{t}\Vert\nabla f(\mathbf{x}_{i}^{t})\Vert^2 + (\sigma_{l}^2 + B^2)\Bigr)
			\end{align*}
			we complete the proof.
		\end{proof}
	\end{lemma}
	
	\begin{lemma}\label{Bounded local error}
		(Bounded local update) Given the stepsize $ 0<\eta_{l} < \frac{1}{6LK}$, we have following bound:
		\begin{equation}
			\frac{1}{m}\sum_{i\in[m]}\mathbb{E}_{t}\Vert\mathbf{x}_{i,k}^{t}-\mathbf{x}_{i}^{t}\Vert^2 \leq 16\eta_{l}^{2}K^2C_0
		\end{equation}
		where $C_0 = \Bigl(\left(3\sigma_{g}^{2}+\sigma_{l}^{2}\right) + 3\frac{1}{m}\sum_{i\in[m]}\mathbb{E}_{t}\Vert\nabla f(\mathbf{x}_{i}^{t})\Vert^2 + (\sigma_{l}^2 + B^2)\Bigr)$
		\begin{proof}
			We use equation \ref{c_k^t} and $\mathbf{c}^t = \sum_{k=0}^{K-1} \frac{\gamma_k}{\gamma}\mathbf{c}_k^t$ complete the proof.
		\end{proof}
	\end{lemma}
	
	\begin{lemma}\label{bounded_e_x}  
		Let $\{\mathbf{x}^{t}(i)\}_{t \ge 0}$ be generated by DFedADMM for all $i \in \{1,2,...,m\}$ and any learning rate $\frac{1}{6KL} > \eta_l > 0 $, we have following bound:
		\begin{equation}
			\frac{1}{m}\sum_{i=1}^{m}\mathbb{E}_t [\|\mathbf{x}_i^{t} - \overline{\mathbf{x}^{t}} \|^2 ] \leq \frac{2C_2}{(1-\psi)^2},
			\nonumber 
		\end{equation}
		Where $ C_2 = 16\eta_{l}^2K^2\left(3\sigma_g^2 +2\sigma_{l}^2 +4B^2\right) + \lambda^2 \left(\sigma_{l}^2 + B^2\right) $.
		\begin{proof}
			Following [Lemma 4, \cite{Sun2022Decentralized}], 
			we denote ${\bf Z}^{t}:=\begin{bmatrix}
				{\bf z}_1^{t},  {\mathbf z}_2^{t},
				\ldots,
				{\bf z}_m^{t}
			\end{bmatrix}^{\top}\in\mathbb{R}^{m\times d}$.
			With these notation, we have
			\begin{align}\label{xtglobal}
				{\bf X}^{t+1}={\bf W}{\bf Z}^{t}={\bf W}{\bf X}^{t}-{\bf \zeta}^t,
			\end{align}
			where ${\bf \zeta}^t:={\bf W}{\bf X}^{t}-{\bf W}{\bf Z}^{t}$.
			The iteration equation (\ref{xtglobal}) can be rewritten as the following expression
			\begin{align}\label{xtglobal2}
				{\bf X}^{t}={\bf W}^{t}{\bf X}^{0}-\sum_{j=0}^{t-1}{\bf W}^{t-1-j}{\bf \zeta}^j.
			\end{align}
			Obviously, it follows
			\begin{equation}\label{trans}
				{\bf W} {\bf P}= {\bf P} {\bf W}={\bf P}.
			\end{equation}
			According to Lemma \ref{mi}, it holds
			$$\|{\bf W}^t-{\bf P}\|\leq \psi^t.$$ Multiplying both sides of equation (\ref{xtglobal2}) with  ${\bf P}$ and using equation (\ref{trans}), we then get
			\begin{align}\label{xtglobal3}
				{\bf P}{\bf X}^{t}={\bf P}{\bf X}^{0}-\sum_{j=0}^{t-1}{\bf P}{\bf \zeta}^j=-\sum_{j=0}^{t-1}{\bf P}{\bf \zeta}^j,
			\end{align}
			where we used initialization ${\bf X}^{0}=\textbf{0}$.
			Then, we are led to
			\begin{equation}\label{xtglobal4}
				\begin{aligned}
					&\|{\bf X}^{t}-{\bf P}{\bf X}^{t}\|=\|\sum_{j=0}^{t-1}({\bf P}-{\bf W}^{t-1-j}){\bf \zeta}^j\|\\
					&\leq \sum_{j=0}^{t-1}\|{\bf P}-{\bf W}^{t-1-j}\|_{\textrm{op}}\|{\bf \zeta}^j\|\leq \sum_{j=0}^{t-1}\psi^{t-1-j}\|{\bf \zeta}^j\|.
				\end{aligned}
			\end{equation}
			With Cauchy inequality,
			\begin{align*}
				&\mathbb{E}\|{\bf X}^{t}-{\bf P}{\bf X}^{t}\|^2\leq \mathbb{E}(\sum_{j=0}^{t-1}\psi^{\frac{t-1-j}{2}}\cdot \psi^{\frac{t-1-j}{2}}\|{\bf \zeta}^j\|)^2\\
				&\leq (\sum_{j=0}^{t-1}\psi^{t-1-j})(\sum_{j=0}^{t-1} \psi^{t-1-j}\mathbb{E}\|{\bf \zeta}^j\|^2)
			\end{align*}
			Direct calculation gives us
			$$\mathbb{E}\|{\bf \zeta}^j\|^2\leq \|{\bf W}\|^2\cdot\mathbb{E}\|{\bf X}^{j}-{\bf Z}^{j}\|^2\leq \mathbb{E}\|{\bf X}^{j}-{\bf Z}^{j}\|^2.$$
			Next,we will bound $ \mathbb{E}\|{\bf X}^{t}-{\bf Z}^{t}\|^2 $.
			\begin{align*}
				\mathbb{E}\|{\bf X}^{t}-{\bf Z}^{t}\|^2
				& = \sum_{i\in[m]}\mathbb{E}_t\Vert\mathbf{x}_i^t - \mathbf{z}_i^t\Vert^2 \\
				&\overset{(a)}= \sum_{i\in[m]}\mathbb{E}_t\Vert\mathbf{x}_i^t - \mathbf{x}_{i,K}^t + \lambda \hat{\mathbf{g}}_i^{t-1}\Vert^2\\
				&\leq2\sum_{i\in[m]}\mathbb{E}_t\Vert\mathbf{x}_i^t - \mathbf{x}_{i,K}^t\Vert^2 + 2\lambda^2\sum_{i\in[m]}\mathbb{E}_t\Vert \hat{\mathbf{g}}_i^{t-1}\Vert^2\\
				&\leq 2m\mathbf{c}_t^K + 2\lambda^2m\frac{1}{m}\sum_{i\in[m]}\mathbb{E}_t\Vert \hat{\mathbf{g}}_i^{t-1}\Vert^2\\
				&\overset{(b)}\leq 32m\eta_{l}^2K^2\left(3\sigma_g^2 +2\sigma_{l}^2 +4B^2\right) + 2\lambda^2 m\left(\sigma_{l}^2 + B^2\right)
			\end{align*}
			Where (a) we use algorithm\ref{Combined-DFedADMM} line 17, (b) apply lemma\ref{Bounded local error} and \ref{Bounded local update}.
			
			Thus, we get:
			\begin{align*}
				\mathbb{E}\|{\bf X}^{t}-{\bf Z}^{t}\|^2 &\leq2m\left(16\eta_{l}^2K^2\left(3\sigma_g^2 +2\sigma_{l}^2 +4B^2\right) + \lambda^2 \left(\sigma_{l}^2 + B^2\right)\right)
			\end{align*}
			In the end ,we get:
			\begin{align*}
				\mathbb{E}\|{\bf X}^{t}-{\bf P}{\bf X}^{t}\|^2
				\leq \frac{2mC_2}{(1-\psi)^2}.
			\end{align*}
			Where $ C_2 = 16\eta_{l}^2K^2\left(3\sigma_g^2 +2\sigma_{l}^2 +4B^2\right) + \lambda^2 \left(\sigma_{l}^2 + B^2\right) $.
			
			The fact that ${\bf X}^{t}-{\bf P}{\bf X}^{t}=\left(
			\begin{array}{c}
				{\bf x}_1^t- \overline{{\bf x}^t}\\
				{\bf x}_1^t- \overline{{\bf x}^t} \\
				\vdots \\
				{\bf x}_m^t- \overline{{\bf x}^t} \\
			\end{array}
			\right)
			$ then proves the result.
		\end{proof}	
	\end{lemma}
	
	\begin{lemma}
		Considering the $\overline{\mathbf{x}^{t}}=\frac{1}{m}\sum_{i\in[m]}\mathbf{x}_{i}^{t}$ is the mean averaged parameters among the last iteration of local clients at time $t$, the auxiliary sequence $\bigl\{ \mathbf{w}^{t}=\overline{\mathbf{x}^{t}}+\frac{1-\gamma}{\gamma}(\overline{\mathbf{x}^{t}}-\overline{\mathbf{x}^{t-1}})\bigr\}_{t>0}$ satisfies the update rule as:
		\begin{equation}\label{w^t}
			\mathbf{w}^{t+1} = \mathbf{w}^{t} - \lambda\frac{1}{m}\sum_{i\in[m]}\sum_{k=0}^{K-1}\frac{\gamma_{k}}{\gamma}\mathbf{g}_{i,k}^{t}.
		\end{equation}
		\begin{proof}
			Firstly, according to the lemma~\ref{Delta} and ${\bf P}{\bf X}^{t+1}={\bf P}{\bf W}{\bf Z}^{t}={\bf P}{\bf Z}^{t}$, that is also
			$\overline{{\bf x}^{t+1}}=\overline{{\bf z}^{t}},$ so we have:
			\begin{align*}
				\overline{\mathbf{x}^{t+1}} - \overline{\mathbf{x}^{t}}
				&=
				\overline{{\bf z}^{t}} - \overline{\mathbf{x}^{t}}\\
				&= \frac{1}{m}\sum_{i\in[m]}(\mathbf{z}_{i}^{t}-\mathbf{x}_{i}^{t})\\
				&= \frac{1}{m}\sum_{i\in[m]}(\mathbf{x}_{i,K}^{t}-\mathbf{x}_{i,0}^{t}-\lambda\hat{\mathbf{g}}_{i}^{t-1})\\
				&= \frac{1}{m}\sum_{i\in[m]}(-\lambda\gamma\sum_{k=0}^{K-1}\frac{\gamma_{k}}{\gamma}\mathbf{g}_{i,k}^{t}+\lambda\gamma\hat{\mathbf{g}}_{i}^{t}-\lambda\hat{\mathbf{g}}_{i}^{t-1})\\
				&= -\lambda\frac{1}{m}\sum_{i\in[m]}\sum_{k=0}^{K-1}\frac{\gamma_{k}}{\gamma}\left(\gamma\mathbf{g}_{i,k}^{t} + (1-\gamma)\hat{\mathbf{g}}_{i}^{t-1}\right).\\
			\end{align*}
			
			Then we expand the the auxiliary sequence $\mathbf{z}^{t}$ as:
			\begin{align*}
				\mathbf{w}^{t+1} - \mathbf{w}^{t}
				&= (\overline{\mathbf{x}^{t+1}}-\overline{\mathbf{x}^{t}}) + \frac{1-\gamma}{\gamma}(\overline{\mathbf{x}^{t+1}}-\overline{\mathbf{x}^{t}}) - \frac{1-\gamma}{\gamma}(\overline{\mathbf{x}^{t}}-\overline{\mathbf{x}^{t-1}})\\
				&= \frac{1}{\gamma}(\overline{\mathbf{x}^{t+1}}-\overline{\mathbf{x}^{t}}) - \frac{1-\gamma}{\gamma}(\overline{\mathbf{x}^{t}}-\overline{\mathbf{x}^{t-1}})\\
				&= -\lambda\frac{1}{m}\sum_{i\in[m]}\left(\Bigl(\sum_{k=0}^{K-1}\frac{\gamma_{k}}{\gamma}\mathbf{g}_{i,k}^{t}\Bigr) + \frac{1-\gamma}{\gamma}\hat{\mathbf{g}}_{i}^{t-1}\right)-\frac{1-\gamma}{\gamma}(\overline{\mathbf{x}^{t}}-\overline{\mathbf{x}^{t-1}})\\
				&= -\lambda\frac{1}{m}\sum_{i\in[m]}\sum_{k=0}^{K-1}\frac{\gamma_{k}}{\gamma}\mathbf{g}_{i,k}^{t} - \frac{1-\gamma}{\gamma}\frac{1}{m}\sum_{i\in[m]}\lambda\hat{\mathbf{g}}_{i}^{t-1}-\frac{1-\gamma}{\gamma}(\overline{\mathbf{x}^{t}}-\overline{\mathbf{x}^{t-1}})\\
				&= -\lambda\frac{1}{m}\sum_{i\in[m]}\sum_{k=0}^{K-1}\frac{\gamma_{k}}{\gamma}\mathbf{g}_{i,k}^{t} -\frac{1-\gamma}{\gamma}\frac{1}{m}\sum_{i\in[m]}(\overline{\mathbf{z}^{t-1}}-\overline{\mathbf{x}^{t-1}}+\lambda\hat{\mathbf{g}}_{i}^{t-1})\\
				&= -\lambda\frac{1}{m}\sum_{i\in[m]}\sum_{k=0}^{K-1}\frac{\gamma_{k}}{\gamma}\mathbf{g}_{i,k}^{t} -\frac{1-\gamma}{\gamma}\frac{1}{m}\sum_{i\in[m]}(\mathbf{z}_{i}^{t-1}-\mathbf{x}_{i}^{t-1}+\lambda\hat{\mathbf{g}}_{i}^{t-1})\\
				&= -\lambda\frac{1}{m}\sum_{i\in[m]}\sum_{k=0}^{K-1}\frac{\gamma_{k}}{\gamma}\mathbf{g}_{i,k}^{t} -\frac{1-\gamma}{\gamma}\frac{1}{m}\sum_{i\in[m]}(\mathbf{x}_{i,K}^{t-1}-\mathbf{x}_{i,0}^{t-1}+\lambda\hat{\mathbf{g}}_{i}^{t-1}-\lambda\hat{\mathbf{g}}_{i}^{t-2})\\
				&= -\lambda\frac{1}{m}\sum_{i\in[m]}\sum_{k=0}^{K-1}\frac{\gamma_{k}}{\gamma}\mathbf{g}_{i,k}^{t}.\\
			\end{align*}
		\end{proof}
	\end{lemma}
	\begin{lemma}\label{10}
		Considering the $\overline{\mathbf{x}^{t}}=\frac{1}{m}\sum_{i\in[m]}\mathbf{x}_{i}^{t}$ is the mean averaged parameters among the last iteration of local clients at time $t$, the auxiliary sequence $\bigl\{ \mathbf{w}^{t}=\overline{\mathbf{x}^{t}}+\frac{1-\gamma}{\gamma}(\overline{\mathbf{x}^{t}}-\overline{\mathbf{x}^{t-1}})\bigr\}_{t>0}$ satisfies the following inequation:
		\begin{align*}
			\frac{1}{m}\sum_{i \in [m]}\mathbb{E}_t\Vert \mathbf{w}^t - \overline{\mathbf{x}^t}\Vert^2\leq
			2\lambda^2\left(\gamma -1\right)^2\left(1+\frac{(\gamma - 1)^2}{\gamma^2}\right)\left(\sigma_{l}^2 + B^2\right)
		\end{align*}
		\begin{proof}
			According to the defination of $\mathbf{w}^t$ ,we have:
			\begin{align*}
				\overline{\mathbf{x}^t} - \mathbf{w}^t &= \frac{\gamma -1}{\gamma}\left(\overline{\mathbf{x}^t} - \overline{\mathbf{x}^{t-1}}\right)\\
				& = \frac{\gamma -1}{\gamma}\left(\overline{\mathbf{z}^{t-1}} - \overline{\mathbf{x}^{t-1}}\right)\\
				& = \frac{\gamma -1}{\gamma}\frac{1}{m}\sum_{i \in [m]}\left(\mathbf{z}_i^{t-1} - \mathbf{x}_i^{t-1}\right)\\
				& = \frac{\gamma -1}{\gamma}\frac{1}{m}\sum_{i \in [m]}\left(\mathbf{x}_{i,K}^{t-1} - \lambda \hat{\mathbf{g}}_i^{t-2} - \mathbf{x}_{i,0}^{t-1}\right)\\
				& = \frac{\gamma -1}{\gamma}\frac{1}{m}\sum_{i \in [m]}\left(-\lambda\gamma\sum_{k=0}^{K-1}\frac{\gamma_{k}}{\gamma}\mathbf{g}_{i,k}^{t-1} + \gamma\lambda\hat{\mathbf{g}}_{i}^{t-2} - \lambda \hat{\mathbf{g}}_i^{t-2}\right)\\
				& = -\lambda(\gamma -1)\frac{1}{m}\sum_{i \in [m]}\sum_{k=0}^{K-1}\frac{\gamma_{k}}{\gamma}\mathbf{g}_{i,k}^{t-1} - \lambda \frac{(\gamma - 1)^2}{\gamma}\frac{1}{m}\sum_{i \in [m]}\hat{\mathbf{g}}_i^{t-2}
			\end{align*}
			so we have:
			\begin{align*}
				\frac{1}{m}\sum_{i \in [m]}\Vert\overline{\mathbf{x}^t} - \mathbf{w}^t\Vert^2 
				&\overset{(a)}\leq 2\lambda^2(\gamma -1)^2(\sigma_{l}^2 + B^2) + 2\lambda^2\frac{(\gamma -1)^4}{\gamma^2}(\sigma_{l}^2 + B^2)\\
				&\leq 2\lambda^2(\gamma - 1)^2\left(1 + \frac{(\gamma - 1)^2}{\gamma^2}\right)(\sigma_{l}^2 + B^2)
			\end{align*}
			where (a) holds according to Lemma\ref{true_global_update}.
		\end{proof}
	\end{lemma}
	
	\subsection{Proof of convergence results for DFedADMM.}\label{conver_proof_DFedADMM}
	
	For the general non-convex case, according to the Assumptions and the smoothness of $f$, we take the conditional expectation at round $t+1$ and expand the $f(\mathbf{w}^{t+1})$ as:
	
	\begin{align}\label{L-smoth}
		\mathbb{E}_t f({\bf w}^{t+1})\leq\mathbb{E}_t f({\bf w}^{t})+\underbrace{\mathbb{E}_t\langle\nabla f({\bf w}^{t+1}),{\bf w}^{t+1}-{\bf w}^{t}\rangle}_{\bf{R2}}
		+\underbrace{\frac{L}{2}\mathbb{E}_t\|{\bf w}^{t+1}-{\bf w}^{t}\|^2}_{\bf{R1}}
	\end{align}
	
	\subsubsection{Bounded R1}
	According to the equation(\ref{w^t}) ,we have
	\begin{align*}
		&\mathbb{E}_t\|{\bf w}^{t+1}-{\bf w}^{t}\|^2
		\leq \lambda^2 \left( \sigma_l^2 + B^2\right)
	\end{align*}
	Therefore
	\begin{align}\label{R1}
		&\frac{L}{2}\mathbb{E}_t\|{\bf w}^{t+1}-{\bf w}^{t}\|^2
		\leq \frac{L}{2}\lambda^2 \left( \sigma_l^2 + B^2\right) 
	\end{align}
	
	\subsubsection{Bounded R2}
	Note that $\mathbf{R2}$ can be bounded as:
	\begin{align*}
		\mathbf{R2} &= \mathbb{E}_t\langle\nabla f({\bf w}^{t}),{\bf w}^{t+1}-{\bf w}^{t}\rangle\\
		& \overset{(\ref{w^t})}= \lambda\mathbb{E}_t\langle \nabla f({\bf w}^{t}), -\frac{1}{m}\sum_{i \in [m]}\sum_{k=0}^{K-1}\frac{\gamma_{k}}{\gamma}\mathbf \nabla f_i(\Breve{\mathbf{x}}_{i,k}^t) \rangle\\
		&= \lambda  \mathbb{E}_t\langle \nabla f({\bf w}^{t}), -\frac{1}{m}\sum_{i \in [m]}\sum_{k=0}^{K-1}\frac{\gamma_{k}}{\gamma}\mathbf \nabla f_i(\Breve{\mathbf{x}}_{i,k}^t) +\nabla f({\bf w}^{t}) -  \nabla f({\bf w}^{t}) \rangle\\
		& =  \underbrace{\lambda  \mathbb{E}_t\langle \nabla f({\bf w}^{t}), -\frac{1}{m}\sum_{i \in [m]}\sum_{k=0}^{K-1}\frac{\gamma_{k}}{\gamma}\mathbf \nabla f_i(\Breve{\mathbf{x}}_{i,k}^t) + \frac{1}{m}\sum_{i \in [m]}\nabla f_i({\bf w}^{t})  \rangle}_{\mathbf{R2.1}} -\lambda  \mathbb{E}_t\Vert \nabla f({\bf w}^{t}) \Vert^2\\
	\end{align*}
	Now we will bound $\mathbf{R2.1}$:
	\begin{align*}
		\mathbf{R2.1} & = \lambda  \mathbb{E}_t\langle \nabla f({\bf w}^{t}), -\frac{1}{m}\sum_{i \in [m]}\sum_{k=0}^{K-1}\frac{\gamma_{k}}{\gamma}\left(\nabla f_i(\Breve{\mathbf{x}}_{i,k}^t) -\nabla f_i({\bf w}^{t})  \right) \rangle \\
		& =-\lambda\mathbb{E}_t\langle \nabla f({\bf w}^{t}), \frac{1}{m}\sum_{i \in [m]}\sum_{k=0}^{K-1}\frac{\gamma_{k}}{\gamma}\left(\mathbf \nabla f_i(\Breve{\mathbf{x}}_{i,k}^t) -\nabla f_i({\bf w}^{t})\right) \rangle\\
		&\overset{(a)}\leq \lambda\left(\frac{1}{2}\mathbb{E}_t\Vert\nabla f({\bf w}^{t})\Vert^2 + \frac{1}{2}\mathbb{E}_t\Vert \frac{1}{m}\sum_{i \in [m]}\sum_{k=0}^{K-1}\frac{\gamma_{k}}{\gamma}\left(\mathbf \nabla f_i(\Breve{\mathbf{x}}_{i,k}^t) -\nabla f_i({\bf w}^{t})\right)\Vert^2 \right)\\
		&\leq\lambda\left(\frac{1}{2}\mathbb{E}_t\Vert\nabla f({\bf w}^{t})\Vert^2 +  \frac{1}{2}\frac{1}{m}\sum_{i \in [m]}\sum_{k=0}^{K-1}\frac{\gamma_{k}}{\gamma}\mathbb{E}_t\Vert\mathbf \nabla f_i(\Breve{\mathbf{x}}_{i,k}^t) -\nabla f_i({\bf w}^{t})\Vert^2\right)\\
		&\leq \lambda \frac{1}{2m}\sum_{i \in [m]}\sum_{k=0}^{K-1}\frac{\gamma_{k}}{\gamma}\mathbb{E}_t\Vert\mathbf \nabla f_i(\Breve{\mathbf{x}}_{i,k}^t) -\nabla f_i(\mathbf{x}_{i}^t) +\nabla f_i(\mathbf{x}_{i}^t)-\nabla f_i(\overline{\mathbf{x}^{t}})+\nabla f_i(\overline{\mathbf{x}^{t}})  -\nabla f_i({\bf w}^{t})\Vert^2 \\
		&\quad + \frac{1}{2}\lambda\mathbb{E}\Vert \nabla f({\bf w}^{t}) \Vert^2 \\
		&\leq 3\lambda L^2 \frac{1}{m}\sum_{i \in [m]}\sum_{k=0}^{K-1}\frac{\gamma_{k}}{\gamma}\mathbb{E}_t\Vert \Breve{\mathbf{x}}_{i,k}^t - \mathbf{x}_i^t\Vert^2 + 3\lambda L^2 \frac{1}{m}\sum_{i \in [m]}\mathbb{E}_t\Vert \mathbf{x}_{i}^t - \overline{\mathbf{x}^t}\Vert^2+
		3\lambda L^2 \frac{1}{m}\sum_{i \in [m]}\mathbb{E}_t\Vert \mathbf{w}^t - \overline{\mathbf{x}^t}\Vert^2\\
		&\quad + \frac{1}{2}\lambda\mathbb{E}\Vert \nabla f({\bf w}^{t}) \Vert^2\\
		&\leq 6\lambda L^2 \frac{1}{m}\sum_{i \in [m]}\sum_{k=0}^{K-1}\frac{\gamma_{k}}{\gamma}\mathbb{E}_t\Vert \mathbf{x}_{i,k}^t - \mathbf{x}_i^t\Vert^2 + 3\lambda L^2 \frac{1}{m}\sum_{i \in [m]}\mathbb{E}_t\Vert \mathbf{x}_{i}^t - \overline{\mathbf{x}^t}\Vert^2+
		3\lambda L^2 \frac{1}{m}\sum_{i \in [m]}\mathbb{E}_t\Vert \mathbf{w}^t - \overline{\mathbf{x}^t}\Vert^2\\
		&\quad + \frac{1}{2}\lambda\mathbb{E}\Vert \nabla f({\bf w}^{t}) \Vert^2 + 6\rho^2\lambda L^2 (\sigma_{l}^2 + B^2)\\
		&\leq 6\lambda L^2 \mathbf{c}^t+3\lambda L^2 \frac{1}{m}\sum_{i \in [m]}\mathbb{E}_t\Vert \mathbf{x}_{i}^t - \overline{\mathbf{x}^t}\Vert^2+
		3\lambda L^2 \frac{1}{m}\sum_{i \in [m]}\mathbb{E}_t\Vert \mathbf{w}^t - \overline{\mathbf{x}^t}\Vert^2 + \frac{1}{2}\lambda\mathbb{E}\Vert \nabla f({\bf w}^{t}) \Vert^2 + 6\rho^2\lambda L^2 (\sigma_{l}^2 + B^2)
	\end{align*}
	(a) applies $-\langle\mathbf{x},\mathbf{y}\rangle\leq\frac{1}{2}\bigl(\Vert\mathbf{x}\Vert^{2}+\Vert\mathbf{y}\Vert^{2}\bigr)$.According to Lemma\ref{bounded_e_x},\ref{Bounded global error} and \ref{10}, $\mathbf{R2.1}$ can be simplified to the following expression:
	\begin{align*}
		\mathbf{R2.1}&\leq 96\lambda L^2\eta_l^2K^2(3\sigma_g^2 + 2\sigma_l^2 + B^2) + 288\lambda L^2\eta_{l}^2 K^2\frac{1}{m}\sum_{i \in [m]}\mathbb{E}_t\Vert \nabla f(\mathbf{x}_i^t)\Vert^2 + 6\lambda L^2 \frac{C_2}{(1-\psi)^2}\\
		&\quad + 6\lambda ^3  L^2 (\gamma - 1)^2\left(1 + (\frac{\gamma - 1}{\gamma})^2\right)(\sigma_{l}^2 + B^2)+\frac{1}{2}\lambda\mathbb{E}\Vert \nabla f({\bf w}^{t}) \Vert^2 + 6\rho^2\lambda L^2 (\sigma_{l}^2 + B^2)\\
	\end{align*}
	We have now completed the upper bound estimation for $ \mathbf{R2.1} $. Then $ \mathbf{R2} $ can be simplified to the following expression:
	\begin{align*}
		\mathbf{R2}&\leq 96\lambda L^2\eta_l^2K^2(3\sigma_g^2 + 2\sigma_l^2 + B^2) + 288\lambda L^2\eta_{l}^2 K^2\frac{1}{m}\sum_{i \in [m]}\mathbb{E}_t\Vert \nabla f(\mathbf{x}_i^t)\Vert^2 + 6\lambda L^2 \frac{C_2}{(1-\psi)^2}\\
		&\quad + 6\lambda ^3  L^2 (\gamma - 1)^2\left(1 + (\frac{\gamma - 1}{\gamma})^2\right)(\sigma_{l}^2 + B^2) - \frac{1}{2}\lambda\mathbb{E}\Vert \nabla f({\bf w}^{t}) \Vert^2 + 6\rho^2\lambda L^2 (\sigma_{l}^2 + B^2)\\
	\end{align*}
	
	\subsubsection{Bounded Global Gradient}
	As we have bounded the term $\mathbf{R1}$ and $\mathbf{R2}$, according to the smoothness inequality, we combine the inequalities above and get the inequality:
	\begin{align*}
		\mathbb{E}_t f({\bf w}^{t+1})&\leq\mathbb{E}_t f({\bf w}^{t})+ \mathbf{R1} + \mathbf{R2}\\
		&\leq \mathbb{E}_t f({\bf w}^{t})+ 96\lambda L^2\eta_l^2K^2(3\sigma_g^2 + 2\sigma_l^2 + B^2) + 288\lambda L^2\eta_{l}^2 K^2\frac{1}{m}\sum_{i \in [m]}\mathbb{E}_t\Vert \nabla f(\mathbf{x}_i^t)\Vert^2 \\
		&\quad + 6\lambda ^3  L^2 (\gamma - 1)^2\left(1 + (\frac{\gamma - 1}{\gamma})^2\right)(\sigma_{l}^2 + B^2) - \frac{1}{2}\lambda\mathbb{E}\Vert \nabla f({\bf w}^{t}) \Vert^2 +\frac{L}{2}\lambda^2 \left( \sigma_l^2 + B^2\right) \\
		&\quad + 6\lambda L^2 \frac{C_2}{(1-\psi)^2}+ 6\rho^2\lambda L^2 (\sigma_{l}^2 + B^2)\\
	\end{align*}
	Furthermore, with \textbf{Lemma} \ref{bounded_e_x}, we can get 
	\begin{equation}
		\begin{split}\label{20}
			\frac{1}{m}\sum_{i=1}^{m}\mathbb{E}_t\left \|\nabla f(\mathbf{x}_i^{t}) \right\|^2 &\leq  2L^2\frac{\sum_{i=1}^{m}\left \| \mathbf{x}_i^{t}-\overline{\mathbf{x}^{t}} \right\|^2}{m} + 2\mathbb{E}_t\left\|\nabla  f(\overline{\mathbf{x}^{t}}) -\nabla f({\bf w}^{t}) + \nabla f({\bf w}^{t})\right\|^2\\
			& {\leq} 4L^2 \frac{C_2}{(1-\psi)^2} +4L^2\mathbb{E}_t\Vert \overline{\mathbf{x}^t} - \mathbf{w}^t\Vert^2 + 4\mathbb{E}_t\left\|\nabla  f({\bf w}^{t})\right\|^2\\
			&\overset{(a)}{\leq} 4L^2 \frac{C_2}{(1-\psi)^2} +8L^2\lambda^2(\gamma - 1)^2\left(1 + \frac{(\gamma - 1)^2}{\gamma^2}\right)(\sigma_{l}^2 + B^2) \\
			&\quad + 4\mathbb{E}_t\left\|\nabla  f({\bf w}^{t})\right\|^2\\
		\end{split}
	\end{equation}
	(a) applies lemma\ref{10}.
	
	Therefore, we have:
	\begin{align*}
		&\mathbb{E}_t f({\bf w}^{t+1})\leq\mathbb{E}_t f({\bf w}^{t})+ \mathbf{R1} + \mathbf{R2}\\
		&\leq \mathbb{E}_t f({\bf w}^{t}) - (\frac{1}{2} - 1152 L^2 \eta_{l}^2 K^2)\lambda\mathbb{E}\Vert \nabla f({\bf w}^{t}) \Vert^2 + 96\lambda L^2\eta_l^2K^2(3\sigma_g^2 + 2\sigma_l^2 + B^2)  \\
		&\quad + 6\lambda ^3  L^2(1 + 384L^2\eta_{l}^2K^2) (\gamma - 1)^2\left(1 + (\frac{\gamma - 1}{\gamma})^2\right)(\sigma_{l}^2 + B^2)+ \frac{L}{2}\lambda^2 \left( \sigma_l^2 + B^2\right) \\
		&\quad + 6\lambda L^2 \left(1+ 192L^2\eta_{l}^2 K^2\right)\frac{C_2}{(1-\psi)^2} + 6\rho^2\lambda L^2 (\sigma_{l}^2 + B^2) \\
		&\overset{(a)}\leq \mathbb{E}_t f({\bf w}^{t}) - (\frac{1}{2} - 1152 L^2 \eta_{l}^2 K^2)\lambda\mathbb{E}\Vert \nabla f({\bf w}^{t}) \Vert^2 + 96\lambda L^2\eta_l^2K^2(3\sigma_g^2 + 2\sigma_l^2 + B^2)  \\
		&\quad + 3\lambda ^3  L^2(1+384L^2\eta_{l}^2K^2) (\sigma_{l}^2 + B^2)+ \frac{L}{2}\lambda^2 \left( \sigma_l^2 + B^2\right) \\
		&\quad + 6\lambda L^2 \left(1 + 192L^2\eta_{l}^2 K^2\right)\frac{C_2}{(1-\psi)^2} + 6\rho^2\lambda L^2 (\sigma_{l}^2 + B^2)\\
	\end{align*}
	Where (a) uses following condition to bound $\gamma$:
	
	Firstly we should note that $\gamma=1-(1-\frac{\eta_{l}}{\lambda})^{K}<1$ when $\eta_{l}\leq 2\lambda$. Thus we have $1/\gamma > 1$. When $K$ satisfies that $K\geq \frac{\lambda}{\eta_{l}}$, $(1-\frac{\eta_{l}}{\lambda})^{K}\leq e^{-\frac{\eta_{l}}{\lambda}K}\leq e^{-1}$, and then $\gamma > 1 - e^{-1}$ and $1/\gamma < \frac{e}{e-1}< 2$.
	
	We denote the constant $\lambda\kappa=\lambda(\frac{1}{2} - 1152 L^2\eta_{l}^2K^2)$ and $\kappa$ could be considered as a constant. We can select  $\eta_{l} < \frac{\sqrt{1-2c_1}}{48\sqrt{2}LK} $. Then we can bound the $\kappa=\frac{1}{2} - 1152L^2\eta_{l}^2K^2 > c_1 > 0$, and the term $\frac{1}{\kappa} < \frac{1}{c_1}$ which is a constant upper bound. Specially, we can let $c_1 = \frac{1}{4}$ ,then $\eta_{l} < \frac{1}{96LK} $ and $\frac{1}{\kappa} < 4$. 
	
	We can obtain the following expression by rearranging the terms:
	\begin{align*}
		\lambda\kappa\mathbb{E}_t \Vert \nabla f({\bf w}^{t})\Vert^2&\leq \left(\mathbb{E}_t f({\bf w}^{t}) - \mathbb{E}_t f({\bf w}^{t+1})\right) +96\lambda L^2\eta_l^2K^2(3\sigma_g^2 + 2\sigma_l^2 + B^2)  \\
		&\quad + 3\lambda ^3  L^2(1+384L^2\eta_{l}^2K^2) (\sigma_{l}^2 + B^2)+ \frac{L}{2}\lambda^2 \left( \sigma_l^2 + B^2\right) \\
		&\quad + 6\lambda L^2 \left(1 + 192 L^2\eta_{l}^2 K^2\right)\frac{C_2}{(1-\psi)^2} + 6\rho^2\lambda L^2 (\sigma_{l}^2 + B^2)\\
	\end{align*}
	Take the full expectation and telescope sum on the inequality above and applying the fact that $f^{*} \leq f(\mathbf{x})$ for $\mathbf{x}\in \mathbb{R}^{d}$, we have:
	\begin{align*}
		\frac{1}{T}\sum_{t=0}^{T-1}\mathbb{E}_{t}\Vert\nabla f(\mathbf{w}^{t})\Vert^{2}
		&\leq \frac{f(\mathbf{w}^{0}) - f^*}{\lambda \kappa T} + \frac{96 L^2\eta_l^2K^2}{\kappa}(3\sigma_g^2 + 2\sigma_l^2 + B^2)  \\
		&\quad + \frac{3\lambda ^2  L^2(1+384L^2\eta_{l}^2K^2)}{\kappa} (\sigma_{l}^2 + B^2)+ \frac{L\lambda}{2\kappa}\left( \sigma_l^2 + B^2\right) \\
		&\quad + \frac{6 L^2 \left(1 + 192L^2\eta_{l}^2 K^2\right)}{\kappa}\frac{C_2}{(1-\psi)^2} + \frac{6\rho^2\lambda L^2}{\kappa} (\sigma_{l}^2 + B^2)\\
	\end{align*}
	Where $ C_2 = 16\eta_{l}^2K^2\left(3\sigma_g^2 +2\sigma_{l}^2 +4B^2\right) + \lambda^2 \left(\sigma_{l}^2 + B^2\right) $.
	
	Here we summarize the conditions and some constrains in the above conclusion. Firstly we should note that $\gamma=1-(1-\frac{\eta_{l}}{\lambda})^{K}<1$ when $\eta_{l}\leq 2\lambda$. Thus we have $1/\gamma > 1$. When $K$ satisfies that $K\geq \frac{\lambda}{\eta_{l}}$, $(1-\frac{\eta_{l}}{\lambda})^{K}\leq e^{-\frac{\eta_{l}}{\lambda}K}\leq e^{-1}$, and then $\gamma > 1 - e^{-1}$ and $1/\gamma < \frac{e}{e-1}< 2$. To let $\kappa=\frac{1}{2} - 1152L^2\eta_{l}^2K>\frac{1}{4} > 0$ hold, $\eta_{l}$ satisfy that $0< \eta_{l} < \frac{1}{96KL}$. 
	
	\begin{align*}
		\frac{1}{T}\sum_{t=0}^{T-1}\mathbb{E}_{t}\Vert\nabla f(\mathbf{w}^{t})\Vert^{2}	
		&\leq \frac{f(\mathbf{w}^{0}) - f^*}{\lambda \kappa T} + \frac{96 L^2\eta_l^2K^2}{\kappa}(3\sigma_g^2 + 2\sigma_l^2 + B^2)  \\
		&\quad + \frac{3\lambda ^2  L^2(1 + 384L^2\eta_{l}^2K^2)}{\kappa} (\sigma_{l}^2 + B^2)+ \frac{L\lambda}{2\kappa}\left( \sigma_l^2 + B^2\right) \\
		&\quad + \frac{96 \eta_{l}^2K^2L^2 \left(1+192L^2\eta_{l}^2 K^2\right)\left(3\sigma_g^2 +2\sigma_{l}^2 +4B^2\right)}{\kappa (1-\psi)^2}\\
		&\quad + \frac{6 \lambda^2L^2 \left(1+192L^2\eta_{l}^2 K^2\right)\left(\sigma_{l}^2 + B^2\right)}{\kappa (1-\psi)^2} + \frac{6\rho^2\lambda L^2}{\kappa} (\sigma_{l}^2 + B^2).
	\end{align*}
	
\end{document}